%% file: main.tex
\documentclass[hidelinks,onefignum,onetabnum]{siamart220329}

\usepackage{bbm}
\usepackage{bm}
\usepackage[short]{optidef}
\usepackage{tikz}
\usetikzlibrary{positioning}
\usetikzlibrary{backgrounds}
\usepackage[font=small,labelfont=bf]{caption}
\usepackage{rotating}
\usepackage{float}
\usepackage{wrapfig}
\usepackage{enumerate}
\usepackage{algorithmic}
\usepackage{algorithm}
\usepackage{varwidth}

\DeclareMathOperator*{\argmax}{arg\,max}
\DeclareMathOperator*{\argmin}{arg\,min}

\newcommand{\norm}[1]{\left\lVert#1\right\rVert}
\newcommand{\mc}[1]{\mathcal{#1}}

\newtheorem{assumption}[theorem]{Assumption}
\theoremstyle{plain}
\newtheorem{example}{Example}

\newcounter{relctr} 
\everydisplay\expandafter{\the\everydisplay\setcounter{relctr}{0}} 

\newcommand\labelrel[2]{%
  \begingroup
    \refstepcounter{relctr}%
    \stackrel{\textnormal{(\alph{relctr})}}{\mathstrut{#1}}%
    \originallabel{#2}%
  \endgroup
}
\AtBeginDocument{\let\originallabel\label}

\newcommand{\newstuff}[1]{{\leavevmode\color{black}{#1}}}

\input{ex_shared}

\ifpdf
\hypersetup{
  pdftitle={Occupancy Information Ratio: Infinite-Horizon, Information-Directed, Parameterized Policy Search},
  pdfauthor={Wesley A. Suttle, Alec Koppel, Ji Liu}
}
\fi

\begin{document}

\maketitle

\begin{abstract}
    In this work, we propose an information-directed objective for infinite-horizon reinforcement learning (RL), called the occupancy information ratio (OIR), inspired by the information ratio objectives used in previous information-directed sampling schemes for multi-armed bandits and Markov decision processes as well as recent advances in general utility RL. The OIR, comprised of a ratio between the average cost of a policy and the entropy of its induced state occupancy measure, enjoys rich underlying structure and presents an objective to which scalable, model-free policy search methods naturally apply. Specifically, we show by leveraging connections between quasiconcave optimization and the linear programming theory for Markov decision processes that the OIR problem can be transformed and solved via concave programming methods when the underlying model is known. Since model knowledge is typically lacking in practice, we lay the foundations for model-free OIR policy search methods by establishing a corresponding policy gradient theorem. Building on this result, we subsequently derive REINFORCE- and actor-critic-style algorithms for solving the OIR problem in policy parameter space. Crucially, exploiting the powerful \textit{hidden quasiconcavity} property implied by the concave programming transformation of the OIR problem, we establish finite-time convergence of the REINFORCE-style scheme to global optimality and asymptotic convergence of the actor-critic-style scheme to (near) global optimality under suitable conditions. \newstuff{Finally, we experimentally illustrate the utility of OIR-based methods over vanilla methods in sparse-reward settings, supporting the OIR as an alternative to existing RL objectives.}
\end{abstract}



\input{sections/Introduction.tex}
\input{sections/Problem.tex}
\input{sections/Elements.tex}
\input{sections/Algorithms.tex}
\input{sections/Convergence_Main_Body_NEW.tex}
\input{sections/Experiments.tex}
\input{sections/Conclusion}

\bibliographystyle{siamplain}
\bibliography{refs,references}
\end{document}

%% file: ex_shared.tex
\usepackage{lipsum}
\usepackage{amsfonts}
\usepackage{graphicx}
\usepackage{epstopdf}
\usepackage{algorithmic}
\ifpdf
  \DeclareGraphicsExtensions{.eps,.pdf,.png,.jpg}
\else
  \DeclareGraphicsExtensions{.eps}
\fi

\newsiamremark{remark}{Remark}
\newsiamremark{hypothesis}{Hypothesis}
\crefname{hypothesis}{Hypothesis}{Hypotheses}
\newsiamthm{claim}{Claim}

\headers{Occupancy Information Ratio}{W. A. Suttle, A. Koppel, J. Liu}

\title{Occupancy Information Ratio: Infinite-Horizon, Information-Directed, Parameterized Policy Search 
}

\author{
Wesley A. Suttle\thanks{U.S. Army Research Laboratory (\email{wesley.a.suttle.ctr@army.mil})}
\and Alec Koppel\thanks{J.P. Morgan AI Research (\email{alec.koppel@jpmchase.com})}
\and Ji Liu\thanks{Department of Electrical \& Computer Engineering, Stony Brook University, NY (\email{ji.liu@stonybrook.edu})}
\funding{The research of W. A. Suttle was supported by the U.S. Army Research Laboratory and was accomplished under Cooperative Agreements W911NF-21-2-0127 and W911NF-22-2-0003.
The research of J.~Liu was supported in part by Army Research Laboratory Cooperative
Agreement W911NF-21-2-0098.
A.~Koppel contributed to this work in his previous role with the Army Research Laboratory in Adelphi, MD 20783.}
}

\usepackage{amsopn}


%% file: sections/Introduction.tex
\section{Introduction}

The field of reinforcement learning (RL) \cite{sutton2018reinforcement} has seen many attempts to address the exploration/exploitation trade-off
by incentivizing exploration via additive regularization; the hope is that, with more experience, the agent can improve its exploitation capabilities.
Prior works on \textit{information-directed} solution methods for multi-armed bandits (MABs) \cite{russo2014learning, russo2016information} and Markov decision processes (MDPs)  \cite{lu2023reinforcement} instead seek to address this trade-off by minimizing an \textit{information ratio} objective, defined as the ratio of cost incurred to information acquired.
Importantly, when used as a tool for devising information-directed action-selection schemes, the specific form of these information ratio objectives leads to policies with improved data efficiency and improved regret bounds revealing the dependence of performance on information. \newstuff{Beyond the original works, the advantages of information-ratio objectives have been analyzed in the frequentist bandit setting \cite{kirschner2021asymptotically}, as well as the more general linear partial monitoring setting \cite{kirschner2020information}.}
In the RL setting, however, the same information-theoretic quantities and assumptions on problem structure that make these insights possible also limit the practical utility of the information ratios proposed in \cite{lu2023reinforcement} as tools for guiding action-selection. In particular, the abstract learning targets, representation of cost in terms of regret, and mutual information formulation of information gain of a policy lead to difficulties in devising practical estimation procedures. Moreover, the practical schemes proposed in \cite{lu2023reinforcement} rely on optimizing over the space $\mathcal{D}(\mathcal{A})$ of action distributions at each step, limiting their practical use to the finite action space setting. Due to these issues, the information ratio and its proxies explored in \cite{lu2023reinforcement} suffer from tractability and scalability issues in realistic settings. 

\newstuff{Gaps therefore remain in the theory of information-directed methods under general function approximation. The work \cite{nikolov2018information} proposes a variant of deep Q-learning that optimizes the ratio of Bellman error to a variance surrogate for information gain, with substantial performance gains in practice, suggesting developing the theory of information-directed schemes that can operate with parameterization is a worthy avenue of pursuit.} New proxy objectives that tractably, scalably extend the spirit of the information-directed schemes of \cite{russo2014learning, russo2016information, lu2023reinforcement} to \newstuff{operate with function approximation and exhibit performance guarantees} are therefore required. In order to achieve this, two issues must be addressed. First, in order to overcome the limited scalability inherent in value-based methods, operating in parameter space is required, for which policy gradient methods are most natural \cite{lillicrap2015continuous, schulman2017proximal, haarnoja2018soft}. Recent theoretical progress has also been made in providing global optimality guarantees for policy gradient methods \cite{bhandari2019global, Agarwal_Kakade_Lee_Mahajan_2020, mei2020global, zhang2020variational, bedi2021sample}, strengthening the motivation for pursuing such methods. Second, to address the estimation issues associated with the notions of information gain used in \cite{lu2023reinforcement}, we need a definition of informativeness that is amenable to policy search in parameter space. Occupancy measure entropy has recently been used as an optimization objective  \cite{hazan2019provably, lee2019efficient, zhang2020variational} quantifying the amount of information about the environment that a policy provides through the Kullback–Leibler divergence of its state occupancy measure from a uniform distribution. Motivated by this, in this work we take occupancy measure entropy, or occupancy information, of a policy as the fundamental quantity defining its informativeness.
Based on this definition, we develop and study a new RL objective called the \textit{occupancy information ratio}, or OIR, which captures the exploration/exploitation trade-off as defined by the ratio of long-term average cost to occupancy information of a policy.

\textbf{Main Contributions.} Our main contributions are as follows. (1) We propose a new RL objective, the occupancy information ratio (OIR), that is both inspired by the information ratio objectives of \cite{russo2014learning, russo2016information, lu2023reinforcement} and amenable to solution via policy search. (2) Drawing on connections between quasiconcave optimization and the linear programming theory for MDPs, we derive a concave programming reformulation of the OIR optimization problem over the space of state-action occupancy measures, establishing underlying theory that we exploit to strengthen our subsequent convergence results. (3) We derive an OIR policy gradient theorem, then use it to develop OIR policy gradient algorithms: Information-Directed REINFORCE (ID-REINFORCE) and Information-Directed Actor-Critic (IDAC). (4) We establish corresponding convergence theory with three key results: (i) OIR policy optimization enjoys a powerful \textit{hidden quasiconcavity} property guaranteeing its first-order stationary points are global optima; (ii) the gradient descent scheme underlying ID-REINFORCE enjoys a non-asymptotic, information-dependent convergence rate; (iii) IDAC converges with probability one to (a neighborhood of) a \textit{global} optimum of the OIR problem. \newstuff{(5) We provide experimental results indicating that OIR-based methods are able to outperform vanilla RL methods in sparse-reward settings, providing auxiliary support for the study of the OIR as an independent RL objective.}

\newstuff{It is important to note that, while the technical motivation for the OIR objective stems from balancing explore-exploit issues via connections with the information ratio methods of \cite{russo2014learning, russo2016information, lu2023reinforcement}, our main convergence theory is of an optimization flavor, in the sense that we provide asymptotic and non-asymptotic analysis of algorithms optimizing the OIR objective. An information-theoretic characterization of the resulting policies remains an important, open problem that we leave for future work.}
\newstuff{The key technical challenge in our results lies in handling the fractional form of the OIR objective, which has not been previously addressed in the literature. To overcome this challenge, we first characterize the \emph{quasiconvex} structure of the OIR problem in \S\ref{sec:elements_of_OIR}. Leveraging this structure, especially properties of the perspective transform familiar to the quasiconcave programming literature, we then extend the concave utility analysis of \cite{zhang2020variational} to \emph{quasiconcave} utilities, including the OIR, in \S\S\ref{sec:hidden_concavity}-\ref{sec:conv_rate}. Finally, we extend the asymptotic actor-critic analyses of \cite{bhatnagar2009natural, suttle2021reinforcement} to our IDAC algorithm, taking special care to establish the requisite smoothness properties of the OIR gradient as well as asymptotic negligibility of corresponding, OIR-specific noise and error terms.
}

%% file: sections/Problem.tex
\section{Problem Formulation} \label{sec:problem_formulation}

We now describe our problem setting and formulate the occupancy information ratio objective. We first define an underlying Markov decision process, then formulate the OIR as an objective to be optimized over it.

\subsection{Markov Decision Processes}
Consider an average-cost MDP described by the tuple $(\mathcal{S}, \mathcal{A}, p, c)$, where $\mc{S}$ is the finite state space, $\mc{A}$ is the finite action space, $p : \mc{S} \times \mc{A} \rightarrow \mc{D}(\mc{S})$ is the transition probability kernel mapping state-action pairs to distributions over the state space, and $c : \mc{S} \times \mc{A} \rightarrow \mathbb{R}^+$ is the cost function mapping state-action pairs to positive scalars.
In this setting, at time-step $t$, the agent is in state $s_t$, chooses an action $a_t$ according to a policy $\pi : \mc{S} \rightarrow \mc{D}(\mc{A})$ mapping states to distributions over $\mc{A}$, incurs cost $c(s_t, a_t)$, and then the system transitions into a new state $s_{t+1} \sim p(\cdot | s_t, a_t)$.
Since we are interested in policy gradient methods, we give the following definitions with respect to a parameterized family $\{ \pi_{\theta} : \mc{S} \rightarrow \mc{D}(\mc{A}) \}_{\theta \in \Theta}$ of policies, where $\Theta\subset\mathbb{R}^d$ is some set of permissible policy parameters. Note that analogous definitions apply to any policy $\pi$. For any $\theta \in \Theta$, let $d_{\theta}(s) = \lim_{t \rightarrow \infty} P(s_t = s \ | \ \pi_{\theta})$
%
%
denote the steady-state occupancy measure over $\mc{S}$ induced by $\pi_{\theta}$, which we assume to be independent of the initial start-state.
In addition, let
$\lambda_{\theta}(s, a) = \lim_{t \rightarrow \infty} P(s_t = s, a_t = a\ | \ \pi_{\theta})$
%
%
denote the state-action occupancy measure induced by $\pi_{\theta}$ over $\mc{S} \times \mc{A}$. Notice that $\lambda_{\theta}(s, a) = d_{\theta}(s) \pi_{\theta}(a | s)$. Furthermore, let $J(\theta) = \sum_s d_{\theta}(s) \sum_a \pi_{\theta} (a | s) c(s, a)$
%
%
denote the long-run average cost of using policy $\pi_{\theta}$. Finally, given $\theta$, define the entropy of the state occupancy measure induced by $\pi_{\theta}$ to be $H(d_{\theta}) = - \sum_s d_{\theta}(s) \log d_{\theta}(s).$
This quantity measures how well $\pi_{\theta}$ covers the state space $\mc{S}$ in the long run.

\subsection{Occupancy Information Ratio}
We consider the OIR objective
\begin{equation} \label{eqn:OIR}
    \rho(\theta) = \frac{J(\theta)}{\kappa + H(d_{\theta})},
\end{equation}
where $\kappa > - \min_{\theta} H(d_{\theta})$ is a user-specified constant chosen to ensure that the denominator in \eqref{eqn:OIR} remains strictly positive. 
%
%
%
Given an MDP $(\mc{S}, \mc{A}, p, c)$, our goal is to find a policy parameter $\theta^*$ such that $\pi_{\theta^*}$ minimizes \eqref{eqn:OIR} \textit{over the MDP}, i.e., subject to its costs and dynamics.
As $J(\theta)$ and $H(d_{\theta})$ are both infinite-horizon quantities, we regard \eqref{eqn:OIR} as an \textit{infinite-horizon objective}.

%

\begin{remark*} \label{rem:kappa_for_rewards}
Though we stipulated that $\kappa > - \min_{\theta} H(d_{\theta})$ in the definition of the OIR above, letting $\kappa < - \max_{\theta} H(d_{\theta})$ has an important interpretation as well. When $\kappa < - \max_{\theta} H(d_{\theta})$ and $J(\theta) \geq 0$, for all $\theta \in \Theta$, clearly the OIR $\rho(\theta)$ will always be non-positive. Because of this, minimizing the OIR will in fact minimize the ratio of $-J(\theta)$ to the absolute value $| \kappa + H(d_{\theta}) |$, or, equivalently, it will maximize $J(\theta) / ( | \kappa | - H(d_{\theta}) )$. This allows the OIR framework to accommodate rewards by simply replacing the cost function $c$ in the MDP with a reward function $r$, and choosing $\kappa < - \max_{\theta} H(d_{\theta})$.
\end{remark*}

\newstuff{
\subsection{OIR as a Proxy Objective for Information-Directed Sampling}

In this section we discuss the occupancy information ratio objective as a proxy for the information ratio objective of the information-directed sampling (IDS) scheme proposed in \cite{lu2023reinforcement}.
%
%
The general setting of \cite{lu2023reinforcement} is a sequential decision-making problem where the goal is to balance optimizing a given objective with acquiring information about an abstract \textit{learning target}, $\mc{X}$, through interactions with the environment, all while maintaining and updating some relevant \textit{epistemic state}, $\mc{P}_t$. For example, $\mc{X}$ may denote the optimal policy for the objective or some suitable exploration scheme, while $\mc{P}_t$ could include policy and value function parameters at time $t$. Given some reward $r$, state $s$, and policy $\pi$, let $V_{\pi}(s)$ denote the value function starting from state $s$ of policy $\pi$, and let $Q_{\pi}(s, a)$ denote the state-action value function for $\pi$ starting from $s, a$. Define $V_*(s) = \max_{\pi} V_{\pi}(s), Q_*(s, a) = \max_{\pi} Q_{\pi}(s, a)$, and let $H( \mc{X} | \mc{P}_t )$ denote the conditional entropy, or remaining uncertainty, of the learning target given $\mc{P}_t$. Once the agent has successfully achieved its learning target, $H( \mc{X} | \mc{P}_t )$ will typically be small or zero. Given $\mc{P}_t$, horizon $\tau$, and a candidate policy $\pi$, let $\mc{P}_{t+\tau}$ denote the epistemic state resulting from starting with $\mc{P}_t$ and using $\pi$ for $\tau$ steps. Then $\left[ H(\mc{X} | P_t) - H(\mc{X} | P_{t+\tau}) \right] / \tau$ is the $\tau$-step \textit{information gain} resulting from following $\pi$.
For a given candidate policy $\pi$, the $\tau$-step information ratio of $\pi$ at time $t$ is defined in \cite{lu2023reinforcement} as the ratio of its instantaneous squared shortfall to its $\tau$-step information gain:
\begin{equation}
    \Gamma_{\tau, t}^{\pi} = \frac{ \mathbb{E}_{\pi} \left[ V_*(s_t) - Q_*(s_t, a_t) \right]^2 }{ \left[ H(\mc{X} | P_t) - H(\mc{X} | P_{t+\tau}) \right] / \tau }. \label{eqn:ids_ir}
\end{equation}
%
%
For a candidate $\pi$, \cite{russo2016information, lu2023reinforcement} show\footnote{The regret bound stated here is a simplification, for purposes of exposition, of the more general statements proven in \cite[\S4.4]{lu2023reinforcement}.} that
%
%
\small
$\text{Regret}(T | \pi) =$ $\sum_{t=0}^{T-1} \mathbb{E}_{\pi} \left[ V_*(s_t) - Q_*(s_t, a_t) \right]$ $\leq \sqrt{ H( \mc{X} | P_0 ) \sum_{t=0}^{T-1} \Gamma_{\tau, t}^{\pi} }$.
\normalsize
This bound suggests that, by choosing a policy minimizing \eqref{eqn:ids_ir} at each timestep, overall regret can be minimized, leading to improved data efficiency due to intelligent information acquisition.
However, several factors limit the tractability of the information ratio objective \eqref{eqn:ids_ir}. First, the presence of $V_*, Q_*$ render explicit estimation of the numerator intractable. Similarly, the specific choice of $\mc{X}$, formulation of $P_t$, and choice of $\tau$ make estimation of the denominator difficult. Objective \eqref{eqn:ids_ir} is thus more useful as an archetype for proxy objectives than as an optimization objective itself. Several Q-learning-based schemes using such proxy objectives are accordingly proposed in \cite{lu2023reinforcement}, yet these are inherently restricted to the finite action space setting and the corresponding proxy objectives are not amenable to optimization using policy gradient-based methods, limiting scalability.

We propose the OIR \eqref{eqn:OIR} as a new proxy objective retaining the spirit of \eqref{eqn:ids_ir} while remaining tractable for parameterized policy search.
To obtain \eqref{eqn:OIR} as a proxy for \eqref{eqn:ids_ir}, we recast its components into policy search-friendly terms. We first replace the squared shortfall in the numerator of \eqref{eqn:ids_ir} with the expected average cost, $J(\theta)$, of the candidate policy $\pi_{\theta}$. Though this substitution loses amenability to the regret analyses of \cite{lu2023reinforcement}, it gains practical and theoretical tractability for policy search by eliminating $V_*$ and $Q_*$ and enabling use of the policy gradient theorem.
Next, we recast the information gain in the denominator. To do this, a tractable learning target $\mc{X}$ and epistemic state $P_t$ must be designed. Without prior knowledge of $\pi^*$, environmental exploration is a natural choice for $\mc{X}$. As discussed in the introduction, state occupancy measure entropy, $H(d_{\theta})$, is widely used to quantify the exploration achieved by policy $\pi_{\theta}$. With $\mc{X}$ denoting state occupancy measure entropy and letting $P_t = \pi_{\theta_t}$, we thus have the interpretation $H(\mc{X} | P_t) = H(d_{\theta_t})$. Similarly, taking $\tau = 1$ for simplicity and letting $P_{t+1} = \pi_{\theta}$ denote the candidate policy, we furthermore have $H(\mc{X} | P_{t+\tau}) = H(d_{\theta})$. This yields our initial information gain reformulation, $\left[ H(\mc{X} | P_t) - H(\mc{X} | P_{t+\tau}) \right] / \tau = H(d_{\theta_t}) - H(d_{\theta})$. Upon closer inspection, however, a slight modification is required: since our goal in the learning target $\mc{X}$ is to increase exploration, we wish to increase $H(d_{\theta})$. We thus redefine $H(\mc{X} | P_t) = -H(d_{\theta_t})$ and $H(\mc{X} | P_{t+\tau}) = -H(d_{\theta})$, yielding $\left[ H(\mc{X} | P_t) - H(\mc{X} | P_{t+\tau}) \right] / \tau = H(d_{\theta}) - H(d_{\theta_t})$, which captures the increase in occupancy measure entropy of candidate policy $\pi_{\theta}$ over the current policy $\pi_{\theta_t}$. Since $H(d_{\theta_t})$ can be viewed as fixed in the expression $J(\theta) / \left[ H(d_{\theta}) - H(d_{\theta_t}) \right]$, we simplify the expression by replacing $H(d_{\theta_t})$ with a constant $\kappa > - \min_{\theta} H(d_{\theta})$ to obtain the OIR of equation \eqref{eqn:OIR}. Despite losing some of the nuance present in the information gain term of \eqref{eqn:ids_ir}, the resulting OIR objective is far more tractable for policy search, as seen in the following section.
}

%% file: sections/Elements.tex
\section{Elements of OIR Optimization} \label{sec:elements_of_OIR}

We now turn to the problem of optimizing the OIR defined in \eqref{eqn:OIR}.
%
First, we build on parallels with linear programming solutions to MDPs and the theory of quasiconcave programming to show that we can transform the non-convex problem of minimizing our objective \eqref{eqn:OIR} into a concave program over the space of state-action occupancy measures.
%
%
%
This endows the OIR optimization problem with a powerful \textit{hidden quasiconcavity} property [cf. \S\ref{sec:hidden_concavity}] that we exploit to strengthen the convergence results for our policy gradient algorithms in \S\ref{sec:theoretical_results}.
Second, 
we lay the groundwork for model-free policy search methods developed in \S\ref{sec:algorithms} by deriving a policy gradient theorem for $\nabla \rho(\theta)$.
%
%
On the road to this result, we derive an \textit{entropy gradient theorem} providing a simple expression for $\nabla H(d_{\theta})$ that we believe is of independent interest.
%

\input{sections/Reformulation.tex}

\input{sections/Policy_Gradients.tex}

%% file: sections/Reformulation.tex
\subsection{Concave Reformulation} \label{sec:concave_program_2}


Given an average-cost MDP $(\mathcal{S}, \mathcal{A}, p, c)$ and a policy $\pi$, let $\lambda_{\pi} \in \mc{D}(\mathcal{S} \times \mathcal{A})$ denote the state-action occupancy measure induced by $\pi$ on $\mathcal{S} \times \mathcal{A}$, i.e., $\lambda_{sa} = \lim_{t \rightarrow \infty} P(s_t = s, a_t = a \ | \pi)$. As discussed in \S8.8 of \cite{Puterman_2014}, if we have access to $p$ and $c$, an optimal state value function can be obtained by solving a related linear program, (P). This is useful, as 
the existence of weakly polynomial-time algorithms for solving linear programs \cite{khachiyan1979polynomial, karmarkar1984new} ensures the problem can be solved efficiently.
Furthermore, the state-action occupancy measure $\lambda^*$ of the optimal policy $\pi^*$ for $(\mc{S}, \mc{A}, p, c)$ can be obtained by solving the following linear program, which is dual to (P): $\min_{\lambda \geq 0} \{ J(\lambda) = c^T \lambda \ | \ \sum_{s, a} \lambda_{sa} = 1 \text{ and } \sum_a \lambda_{sa} = \sum_{s', a} p(s | s', a) \lambda_{s'a}, \forall s \in \mathcal{S} \}$. Call this dual linear program $(D)$.
%
%
The constraints ensure that the decision variables $\lambda$ give a valid state-action occupancy measure for the MDP.
So, given a feasible solution $\lambda$ to (D), $J(\lambda)$ is clearly the expected long-run average cost of following a policy that induces $\lambda$. Furthermore, the policy $\pi_{\lambda}$ defined by
$\pi_{\lambda}(a | s) = \frac{\lambda_{sa}}{\sum_{a'} \lambda_{sa'}}$
induces $\lambda$ (see Thm. 8.8.2 in \cite{Puterman_2014}). This means that, once the optimal $\lambda^*$ is obtained by solving (D), the corresponding policy $\pi_{\lambda^*}$ is optimal for $(\mc{S}, \mc{A}, p, c)$.

An analogous problem can be formulated for minimizing \eqref{eqn:OIR} over $(\mc{S}, \mc{A}, p, c)$. Consider the following:
\vspace{-1.5mm}
\small
\begin{mini*}
  {\lambda \geq 0}{\rho(\lambda) = \frac{J(\lambda)}{\kappa + \widehat{H}(\lambda)} }{}{}
  \addConstraint{\sum_{s, a} \lambda_{sa}}{= 1}{}  \tag{$Q$} \label{opt:Q0}
  \addConstraint{\sum_a \lambda_{sa}}{= \sum_{s', a} p(s | s', a) \lambda_{s'a},}{\hspace{4mm} \forall s \in \mathcal{S}}
\end{mini*}
\normalsize
%
where $\widehat{H}(\lambda) = H(d^{\lambda})$ is the entropy of $d^{\lambda} \in \mc{D}(\mathcal{S})$ given by $d^{\lambda}_s = \sum_a \lambda_{sa}$. In the standard definition of the function $H(d)$, for any $d_i = 0$, we take $d_i \log d_i = \lim_{d_i \rightarrow 0^+} d_i \log d_i = 0$, so $H(d)$ is always well-defined and finite for $d \geq 0$ (see, e.g., \cite{gray2011entropy}). Similarly, we take $d^{\lambda}_s \log d^{\lambda}_s = 0$ whenever $d^{\lambda}_s = 0$, so that $\widehat{H}(\lambda)$ is well-defined for $\lambda \geq 0$. To ensure that the objective of \eqref{opt:Q0} is well-defined over its feasible region, we make the following mild assumption:
\begin{assumption} \label{assum:two_nonzero_entries}
    For all $\lambda$ feasible to \eqref{opt:Q0}, $d^{\lambda}$ has at least two non-zero entries.
\end{assumption}
This ensures $\rho(\lambda_{\pi})$ is well-defined, for any $\pi$, and is weaker than the ergodicity conditions frequently encountered in the RL literature [cf. Assumption \ref{assum:actor_conditions}].

Since the feasible region of \eqref{opt:Q0} corresponds to precisely those state-action occupancy measures achievable over $(\mc{S}, \mc{A}, p, c)$, solving \eqref{opt:Q0} yields the state-action occupancy measure minimizing $\rho(\lambda)$. Furthermore, as with (D) above, any $\lambda^*$ optimal to \eqref{opt:Q0} allows us to recover a policy $\pi_{\lambda^*}$ minimizing $\rho(\lambda)$. Unlike (D), however, the objective function in \eqref{opt:Q0} is non-convex, so the problem may be difficult to solve directly. Fortunately, due to the \textit{quasiconvexity} of $\rho(\lambda)$ [cf. Definition \ref{def:quasiconvex}], the problem \eqref{opt:Q0} can be transformed via the substitution $y = \lambda / c^T \lambda, t = 1 / c^T \lambda$ and an application of the perspective transform [cf. Definition \ref{def:perspective_transform}] described in Chapter 7 of \cite{avriel2010generalized} to the equivalent concave program
\small
\begin{maxi*}
    {y \geq 0, t}{\kappa t - \sum_{s, a} y_{sa} \log \left( \frac{\sum_a y_{sa}}{t} \right)}{}{}
    \addConstraint{\sum_{s, a} y_{sa}}{= t, \quad \sum_{s, a} c_{sa} y_{sa} = 1}{} \tag{$Q'$} \label{opt:Q0_concave}
    \addConstraint{\sum_a y_{sa}}{= \sum_{s', a} p(s | s', a) y_{s'a},}{\hspace{4mm} \forall s \in \mathcal{S}}
\end{maxi*}
\normalsize
This problem can be efficiently solved using well-known methods for concave optimization \cite{boyd2004convex} to obtain the optimal state-occupancy measure and corresponding optimal policy. We formalize this as the following result:
\begin{theorem} \label{thm:Q0_eq_Q0_concave}
Problem \eqref{opt:Q0_concave} is a concave program, and any optimal solution to it is optimal for the OIR problem \eqref{opt:Q0}.
\end{theorem}
In addition to enabling efficient solution when the MDP model is known, this reformulation implies the existence of \textit{hidden quasiconcavity} underlying any policy gradient methods developed for the OIR minimization problem, as shown in \S\ref{sec:hidden_concavity}.
\\ \\
\textbf{Proof of Theorem \ref{thm:Q0_eq_Q0_concave}.} The remainder of this subsection is devoted to the proof of Theorem \ref{thm:Q0_eq_Q0_concave}.
%

\subsubsection{Quasiconvexity of \eqref{opt:Q0}} \label{subsubsec:Q0_quasiconvex}

Let us first formally define quasiconvexity/-concavity. Given a scalar $\alpha$ and function $f : \mathbb{R}^n \rightarrow \mathbb{R}$ defined on a convex set $C \subset \mathbb{R}^n$, define the $\alpha$-superlevel set of $f$ on $C$ to be $U(f, \alpha) = \{ x \in C \ | \ f(x) \geq \alpha \}$ and the $\alpha$-sublevel set of $f$ on $C$ to be $L(f, \alpha) = \{ x \in C \ | \ f(x) \leq \alpha \}$.
\begin{definition} \label{def:quasiconvex}
Given $f : \mathbb{R}^n \rightarrow \mathbb{R}$ defined on a convex set $C \subset \mathbb{R}^n$, $f$ is \textit{quasiconvex (resp. quasiconcave)} if $L(f, \alpha)$ (resp. $U(f, \alpha)$) is convex, for each $\alpha \in \mathbb{R}$.
\end{definition}
\noindent Now let $\Delta(\mathbb{R}^{|\mathcal{S}| \cdot |\mathcal{A}|})$ denote the unit simplex in $\mathbb{R}^{|\mathcal{S}| \cdot |\mathcal{A}|}$, and let $F$ denote the feasible region of \eqref{opt:Q0}. Clearly $F$ is a convex subset of $\Delta(\mathbb{R}^{|\mc{S}| \cdot |\mc{A}|})$, since it is defined by linear equality and nonnegativity constraints. Note that the numerator of $\rho(\lambda)$, the objective function in \eqref{opt:Q0}, is convex (linear, in fact). Also notice that $\widehat{H}(\lambda) = H(d^{\lambda})$ is concave on $F$, which follows from the facts that the entropy $H(d)$ is concave in $d$, $d^{\lambda}$ is a linear function of $\lambda$, and the composition of a concave function with an affine function is itself concave. This implies that, for any fixed $\kappa \geq 0$, the denominator of $\rho(\lambda)$ is concave and also positive by Assumption \ref{assum:two_nonzero_entries} over all its sublevel subsets of the feasible region. These facts guarantee that \eqref{opt:Q0} is a quasiconvex program with a well-behaved objective function, as formalized in the following Lemma.
\begin{lemma} \label{lemma:quasiconvex}
\eqref{opt:Q0} is feasible and has an optimal solution with finite objective function value, and the objective $\rho$ of \eqref{opt:Q0} is strictly quasiconvex on $F$.
\end{lemma}
%
%
%
%
\noindent Finally, \eqref{opt:Q0} enjoys the following key property, which guarantees that any solution to the concave program described in the next section provides a globally optimal solution to the OIR minimization problem \eqref{eqn:OIR}.
\begin{lemma} \label{lemma:local_implies_global}
Every local optimum of \eqref{opt:Q0} is a global optimum.
\end{lemma}
\begin{proof}
The assertion follows directly from Prop.~3.3 in \cite{avriel2010generalized}.
\end{proof}

\subsubsection{Transformation to a Concave Program} \label{subsubsec:transformation}

Now that we are assured \eqref{opt:Q0} is quasiconvex and has no spurious stationary points, we show it can be reformulated as an equivalent concave program by leveraging results from classic results from the literature on quasiconcave programming \cite{schaible1976fractional, avriel2010generalized}.
%
%
As highlighted at the beginning of this section, the definition of the OIR is critical to this reformulation, as it exploits structural attributes of the family of state-action occupancy measures to enable the transformation from the initial quasiconvex program to the desired concave program.

Define $q(\lambda) := 1 / \rho(\lambda) = (\kappa + \widehat{H}(\lambda)) / J(\lambda)$ and consider the problem
\begin{maxi*}
  {\lambda}{q(\lambda)}{}{}
    \addConstraint{\sum_s \sum_a \lambda_{sa}}{= 1}{}  \tag{$Q''$} \label{opt:Q1}
    \addConstraint{\sum_a \lambda_{sa}}{= \sum_{s'} \sum_a p(s | s', a) \lambda_{s'a},}{\hspace{4mm} \forall s \in \mathcal{S}}
    \addConstraint{\lambda}{\geq 0}{}{}.
\end{maxi*}
\noindent Note that the feasible region $F$ of \eqref{opt:Q1} is identical to the feasible region of \eqref{opt:Q0}. We have the following result:
\begin{lemma} \label{lemma:Q0_eq_Q1}
Problem \eqref{opt:Q0} is equivalent to \eqref{opt:Q1}.
\end{lemma}
\begin{proof}
Assume $\lambda^*$ is optimal for \eqref{opt:Q0}, i.e. $\lambda^* \in F$ and $\rho(\lambda^*) \leq \rho(\lambda)$, for all $\lambda \in F$. By Lemma \ref{lemma:quasiconvex}, there exists $M > 0$ such that $0 < \rho(\lambda^*) \leq \rho(\lambda) < M < \infty$, for all $\lambda \in F$. Clearly $0 < 1/M < q(\lambda) \leq q(\lambda^*) < \infty$, for all $\lambda \in F$, so $\lambda^*$ is optimal to \eqref{opt:Q1}. By an analogous argument, any optimal solution to \eqref{opt:Q1} is optimal to~\eqref{opt:Q0}.
\end{proof}
\newstuff{
\noindent The foregoing Lemma proves that solving \eqref{opt:Q1} also solves \eqref{opt:Q0}. Crucially, as shown in Theorem \ref{thm:quasiconcave_to_concave} below, we can in fact transform \eqref{opt:Q1} into a \textit{concave} optimization problem, which will allow us to indirectly solve \eqref{opt:Q0}. Before presenting the theorem, we provide an important definition.
\begin{definition} \label{def:perspective_transform}
Given $f : \mathbb{R}^n \rightarrow \mathbb{R}$, the \textit{perspective} of $f$ is the function $P_f : \mathbb{R}^{n+1} \rightarrow \mathbb{R}$ given by $P_f(x, t) = t f(x / t)$ with domain $\textbf{dom}(P_f) = \{ (x, t) \ | \ x / t \in \textbf{dom}(f), t > 0\}$.
\end{definition}
We now proceed with the theorem, whose proof follows that of \cite[Prop. 7.2]{avriel2010generalized}.
\begin{theorem} \label{thm:quasiconcave_to_concave}
    The quasiconcave program \eqref{opt:Q1} can be converted via the variable transformation $y = \frac{\lambda}{c^T \lambda}, \hspace{2mm} t = \frac{1}{c^T \lambda}$ into the following concave program: 
    \begin{maxi*}
        {y, t}{\kappa t - \sum_s \sum_a y_{sa} \log \left( \frac{\sum_a y_{sa}}{t} \right)}{}{}
        \addConstraint{\sum_s \sum_a y_{sa}}{= t} \tag{$Q'''$} \label{opt:Q3}
        \addConstraint{\sum_a y_{sa}}{= \sum_{s'} \sum_a p(s | s', a) y_{s'a},}{\hspace{4mm} \forall s \in \mathcal{S}}
        \addConstraint{\sum_s \sum_a c_{sa} y_{sa}}{= 1}
        \addConstraint{y}{\geq 0}.
    \end{maxi*}
\end{theorem}
\begin{proof}
    First, the transformation $y = \frac{\lambda}{c^T \lambda}, \hspace{2mm} t = \frac{1}{c^T \lambda}$ clearly provides a bijection between the feasible regions of \eqref{opt:Q1} and \eqref{opt:Q3}. Next, let $f(\lambda) = \kappa + \widehat{H}(\lambda)$ denote the numerator of $q(\lambda) = \big( \kappa + \widehat{H}(\lambda) \big) / c^T \lambda$. It is immediate that $q(\lambda) = t f(y / t)$, and recalling the definition $\widehat{H}(\lambda) = - \sum_{s,a} \lambda_{sa} \log \left( \sum_a \lambda_{sa} \right)$ allows us to see that $t f(y / t) = \kappa t - \sum_{sa} y_{sa} \log \left( \sum_a y_{sa} / t \right)$. The objectives of \eqref{opt:Q1} and \eqref{opt:Q3} thus share the same value for corresponding points in their feasible regions. Since $\widehat{H}(\lambda)$ is concave in $\lambda$, and since the perspective transform of a concave function is itself concave by \cite[\S3.6.2]{boyd2004convex}, the objective $t f(y / t)$ of \eqref{opt:Q3} is concave. Finally, since the feasible region of \eqref{opt:Q3} is determined by linear equalities and positivity constraints, its feasible region is convex. Problem \eqref{opt:Q3} is thus a concave program.
\end{proof}
}
Given an optimal solution $(y^*, t^*)$ to \eqref{opt:Q3}, we can recover an optimal solution $\lambda^* = y^* / t^*$ to \eqref{opt:Q1}, which by Lemma \ref{lemma:Q0_eq_Q1} is also optimal for \eqref{opt:Q0}. As noted above, the policy $\pi^*(a | s) = \lambda^*_{sa} / \sum_{a'} \lambda^*_{sa'}$ thus minimizes the OIR over the MDP $(\mc{S}, \mc{A}, p, c)$. This proves Theorem \ref{thm:Q0_eq_Q0_concave}.

%% file: sections/Policy_Gradients.tex
\subsection{Policy Gradients}

%
%
%
%
Sampling the gradient of \eqref{eqn:OIR} is not straightforward using existing tools, as obtaining stochastic estimates of $\nabla \rho(\theta)$ involves estimating
\begin{equation} \label{eqn:info_ratio_gradient}
    \nabla \rho(\theta) 
    = \frac{\nabla J(\theta) (\kappa + H(d_{\theta})) - J(\theta)  \nabla H(d_{\theta})}{[\kappa + H(d_{\theta})]^2}.
\end{equation}
Though we can use the classical policy gradient theorem [cf. Eq. \eqref{eqn:classic_pgt}] to estimate $\nabla J(\theta)$ and we can empirically estimate $J(\theta)$ and $H(d_{\theta})$, it is not obvious how to estimate $\nabla H(d_{\theta})$. In what follows we prove an \textit{entropy gradient theorem} that allows us to estimate $\nabla H(d_{\theta})$ and consequently $\nabla \rho(\theta)$.


\subsubsection{Policy Gradient Preliminaries}
%
%
Given an MDP $(\mathcal{S}, \mathcal{A}, p, c)$ and policy $\pi_{\theta}$, two important objects from the RL literature are the relative state value function
%
%
$V_{\theta}(s) = \sum_{t = 0}^{\infty} \mathbb{E}_{\pi_{\theta}} \left[ c(s, a) - J(\theta) \ | \ s_0 = s \right]$
and the relative action value function
%
%
$Q_{\theta}(s, a) = \sum_{t = 0}^{\infty} \mathbb{E}_{\pi_{\theta}} \left[ c(s, a) - J(\theta) \ | \ s_0 = s, a_0 = a \right].$
Under the assumption that $\pi_{\theta}(a | s)$ is differentiable in $\theta$, for all $s \in \mc{S}, a \in \mc{A}$, classic policy gradient methods minimize $J(\theta)$ by taking stochastic gradient descent steps in the direction $- \nabla J(\theta)$.
We are guaranteed by the policy gradient theorem \cite{sutton1999policy} that, under certain conditions,
\begin{equation}
    \nabla J(\theta) = \sum_s d_{\theta}(s) \sum_a Q_{\theta}(s, a) \nabla \pi_{\theta}(a | s) = \mathbb{E}_{\pi_{\theta}} \Big[ ( c(s, a) - J(\theta) ) \nabla \log \pi_{\theta}(a | s) \Big]. \label{eqn:classic_pgt}
\end{equation}
By following policy $\pi_{\theta}$, we can sample from the right-hand side of \eqref{eqn:classic_pgt} to estimate $\nabla J(\theta)$, then use this to perform stochastic gradient descent. 

\subsubsection{Cross-Entropy Gradient}
To estimate $\nabla \rho(\theta)$ we must know how to estimate $\nabla H(d_{\theta})$. Fortunately, by using the relationship between entropy and cross-entropy, $\nabla H(d_{\theta})$ can be estimated in a straightforward manner.
Given two policy parameters $\theta$ and $\theta'$,
the cross-entropy between $d_{\theta}$ and $d_{\theta'}$ is given by $CE(d_{\theta}, d_{\theta'}) = - \sum_s d_{\theta}(s) \log d_{\theta'}(s)$ and their Kullback-Leibler (KL) divergence by $D_{KL}(d_{\theta} \ || \ d_{\theta'}) = \sum_s \log \left( \frac{d_{\theta}(s)}{d_{\theta'}(s)} \right) d_{\theta}(s)$.
Recall that $CE(d_{\theta}, d_{\theta'}) = H(d_{\theta}) + D_{KL}(d_{\theta} \ || \ d_{\theta'})$. We have:\footnote{For a function $f : \Theta \rightarrow \mathbb{R}$, we sometimes write $\nabla f(\theta) \rvert_{\theta = \theta_t}$ to emphasize the fact that the gradient of $f$ w.r.t. $\theta$ is being taken first, then subsequently evaluated at $\theta = \theta_t$.}

\begin{lemma} \label{lemma:entropy_cross_entropy_gradient}
For any $\theta' \in \Theta$,
\begin{align} \label{eqn:s_entropy_gradient}
    \nabla H(d_{\theta}) \big\rvert_{\theta = \theta'} = \nabla CE(d_{\theta}, d_{\theta'}) \big\rvert_{\theta = \theta'}.
\end{align}
\end{lemma}

\begin{proof}
Notice that
\begin{align}
    \nabla CE(d_{\theta}, d_{\theta'}) \big\rvert_{\theta = \theta'} &= \nabla \left[ H(d_{\theta}) + D_{KL}(d_{\theta} \ || \ d_{\theta'}) \right] \big\rvert_{\theta = \theta'} \nonumber \\
    &= \nabla H(d_{\theta}) \big\rvert_{\theta = \theta'} + \nabla D_{KL}(d_{\theta} \ || \ d_{\theta'}) \big\rvert_{\theta = \theta'}. \label{eqn:ce_entropy_kl}
\end{align}
Expanding the term $\nabla D_{KL}(d_{\theta} \ || \ d_{\theta'})$, where the gradient is being taken w.r.t. $\theta$ and $\theta'$ is fixed, we get
\small
\begin{align*}
    \nabla D_{KL}( d_{\theta} \ || \ d_{\theta'} ) &= \sum_s \nabla \left[ \log \left( \frac{d_{\theta}(s)}{d_{\theta'}(s)} \right) d_{\theta}(s) \right] = \sum_s \nabla \left[ ( \log d_{\theta}(s) - \log d_{\theta'}(s) ) d_{\theta}(s) \right] \nonumber \\
    &\labelrel={eqn:kl1} \sum_s \left[ \left( \nabla [ \log d_{\theta}(s) - \log d_{\theta'}(s) ] \right) d_{\theta}(s) + [ \log d_{\theta}(s) - \log d_{\theta'}(s) ] \nabla d_{\theta}(s) \right] \nonumber \\
    &\labelrel={eqn:kl2} \sum_s \left[ \frac{\nabla d_{\theta}(s)}{d_{\theta}(s)} d_{\theta}(s) + [ \log d_{\theta}(s) - \log d_{\theta'}(s) ] \nabla d_{\theta}(s) \right] \nonumber \\
    &\labelrel={eqn:kl3} \sum_s \left[ 1 + \log d_{\theta}(s) - \log d_{\theta'}(s) \right] \nabla d_{\theta}(s), 
\end{align*}
\normalsize
where \eqref{eqn:kl1} follows from the product rule, \eqref{eqn:kl2} holds by the chain rule, and the $d_{\theta}(s)$ terms cancel to yield \eqref{eqn:kl3}. Evaluating the last expression at $\theta = \theta'$, we end up with
\small
\begin{align*}
    \nabla D_{KL}(d_{\theta} \ || \ d_{\theta'}) \big\rvert_{\theta = \theta'} &= \left( \sum_s \left[ 1 + \log d_{\theta}(s) - \log d_{\theta'}(s) \right] \nabla d_{\theta}(s) \right) \Bigg\rvert_{\theta = \theta'} \\
    &\labelrel={eqn:kl4} \sum_s \left[ 1 + \log d_{\theta'}(s) - \log d_{\theta'}(s) \right] \left( \nabla d_{\theta}(s) \big\rvert_{\theta = \theta'} \right) \\
    &= \sum_s \left( \nabla d_{\theta}(s) \big\rvert_{\theta = \theta'} \right)
    \labelrel={eqn:kl5} \left( \sum_s \nabla d_{\theta}(s) \right) \Bigg\rvert_{\theta = \theta'}
    \labelrel={eqn:kl6} \left( \nabla \sum_s d_{\theta}(s) \right) \Bigg\rvert_{\theta = \theta'} \\
    &\labelrel={eqn:kl7} \nabla 1 \big\rvert_{\theta = \theta'} = 0,
\end{align*}
\normalsize
where we evaluate the terms $\log d_{\theta}(s)$ and $\nabla d_{\theta}(s)$ at $\theta = \theta'$ to get \eqref{eqn:kl4}, move the evaluation $\theta = \theta'$ outside the summation in \eqref{eqn:kl5}, pull the gradient outside the summation to obtain \eqref{eqn:kl6}, and finally recall that $\sum_s d_{\theta}(s) = 1$ to get \eqref{eqn:kl7}.
But, recalling equation \eqref{eqn:ce_entropy_kl}, this means that $\nabla H(d_{\theta}) \big\rvert_{\theta = \theta'} = \nabla CE(d_{\theta}, d_{\theta'}) \big\rvert_{\theta = \theta'}$,
%
%
completing the proof.
\end{proof}

This establishes an important fact: \textit{we can estimate $\nabla H(d_{\theta}) \rvert_{\theta = \theta_t}$ by instead estimating $\nabla CE(d_{\theta}, d_{\theta_t}) \rvert_{\theta = \theta_t}$.}
At first glance, this simply substitutes one problem for another.
However, given a fixed $\theta_t$, for any $\theta$, we can use the policy gradient theorem \eqref{eqn:classic_pgt} to obtain a tractable expression for $\nabla CE(d_{\theta}, d_{\theta_t}) \rvert_{\theta = \theta_t}$, as described next.

\subsubsection{Entropy and OIR Policy Gradients}
Our next results enable policy gradient algorithms for maximizing $H(d_{\theta})$ and minimizing \eqref{eqn:OIR}.

\begin{theorem}
Let an MDP $(\mathcal{S}, \mathcal{A}, p, c)$ and a differentiable parametrized policy class $\{ \pi_{\theta} \}_{\theta \in \Theta}$ be given, and recall the definition above of the state occupancy measure $d_{\theta}$ induced by $\pi_{\theta}$ on $\mathcal{S}$. Fix a policy parameter iterate $\theta_t$ at time-step $t$. The gradient $\nabla H(d_{\theta}) \rvert_{\theta = \theta_t}$ [cf. \eqref{eqn:s_entropy_gradient}] with respect to the policy parameters $\theta$ of the state occupancy measure entropy $H(d_{\theta})$, evaluated at $\theta = \theta_t$, satisfies
\label{thm:entropy_gradient}
    \begin{equation} \label{eqn:entropy_gradient}
        \nabla H(d_{\theta}) \big\rvert_{\theta = \theta_t} = \mathbb{E}_{\pi_{\theta_t}} \Big[ \left( - \log d_{\theta_t}(s) - H(d_{\theta_t}) \right) \nabla \log \pi_{\theta_t}(a | s) \Big].
        %
    \end{equation}
%
\end{theorem}
\begin{proof}
    At a given $t$, consider the average-reward MDP $(\mathcal{S}, \mathcal{A}, p, r_t)$, where $r_t: \mathcal{S} \rightarrow \mathbb{R}^+$ is the purely state-dependent (i.e., action-independent) reward given by $r_t(s) = -\log d_{\theta_t}(s)$. We refer to the MDP $(\mathcal{S}, \mathcal{A}, p, r_t)$ as the \textit{shadow MDP} associated with parameter $\theta_t$. Define $\mu^t(\theta) = \lim_{n \rightarrow \infty} \frac{1}{n} \mathbb{E} \left[ \sum_{i=1}^n r_t(s_i) \ | \ \pi_{\theta} \right]$ and $Q^t_{\theta}(s, a) = \sum_{n=1}^{\infty} \mathbb{E} \left[ r_t(s_n) - \mu^t(\theta) \ | \ s_0 = s, a_0 = a, \pi_{\theta} \right]$.
    %
    %
    %
    Here $\mu^t(\theta)$ is the expected long-run average reward of using $\pi_{\theta}$ on the shadow MDP $(\mathcal{S}, \mathcal{A}, p, r_t)$, while $Q^t_{\theta}$ is the corresponding state-action value function. Given a policy $\pi_{\theta}$, 
    \small
    \begin{equation*}
    CE(d_{\theta}, d_{\theta_t}) = - \sum_s d_{\theta}(s) \log d_{\theta_t}(s) = \sum_s d_{\theta}(s) \ r_t(s) = \sum_s d_{\theta}(s) \sum_a \pi_{\theta}(a | s) \ r_t(s),
    \end{equation*}
    \normalsize
    where the last equality holds by the fact that $\sum_a \pi_{\theta}(a | s) = 1$, for all $s \in \mathcal{S}$.
    The policy gradient theorem \eqref{eqn:classic_pgt} can thus be used to yield the expression
    \small
    \begin{equation} \label{eqn:CE_pgt}
        \nabla CE(d_{\theta}, d_{\theta_t}) = \sum_s d_{\theta}(s) \sum_a Q^t_{\theta}(s, a) \nabla \log \pi_{\theta}(a | s).
    \end{equation}
    \normalsize
    Combining expression \eqref{eqn:CE_pgt} with Lemma \ref{lemma:entropy_cross_entropy_gradient} and noticing that $\mu^t(\theta_t) = H(d_{\theta_t})$ yields equation \eqref{eqn:entropy_gradient}.
\end{proof}
%
%
%
%
With Theorem \ref{thm:entropy_gradient} in hand, we have the following OIR policy gradient theorem:
\begin{theorem}
Let MDP $(\mathcal{S}, \mathcal{A}, p, c)$, differentiable policy class $\{ \pi_{\theta} \}_{\theta \in \Theta}$, and constant $\kappa \geq 0$ be given, and recall the definitions of the average cost $J(\theta)$, state occupancy measure $d_{\theta}$, and entropy $H(d_{\theta})$. Fix a policy parameter iterate $\theta_t$ at time-step $t$. The gradient $\nabla \rho(\theta_t)$ [cf. \eqref{eqn:info_ratio_gradient}] with respect to the policy parameters $\theta$ of the OIR $\rho(\theta)$ [cf. \eqref{eqn:OIR}], evaluated at $\theta = \theta_t$, satisfies
\label{thm:info_ratio_gradient}
    \small
    \begin{equation} \label{eqn:info_ratio_gradient_thm}
        %
        \nabla \rho(\theta_t) = \mathbb{E}_{\pi_{\theta_t}} \left[ \frac{ \delta_t^J \big( \kappa + H(d_{\theta_t}) \big) - J(\theta_t) \delta_t^H}{ \left[ \kappa + H(d_{\theta_t}) \right]^2} \psi_t \right],
    \end{equation}
    \normalsize
where $\delta_t^J = c(s, a) - J(\theta_t)$, $\delta_t^H = -\log d_{\theta_t}(s) - H(d_{\theta_t})$, and $\psi_t = \nabla \log \pi_{\theta_t}(a | s)$.
\end{theorem}
The claim follows by combining equations \eqref{eqn:info_ratio_gradient} and \eqref{eqn:classic_pgt} with Theorem \ref{thm:entropy_gradient}.
Armed with this, we next develop policy gradient algorithms for minimizing the OIR.

%% file: sections/Algorithms.tex
\section{Algorithms} \label{sec:algorithms}

In this section we derive two policy search schemes for minimizing \eqref{eqn:OIR}. The first is based on the classic REINFORCE algorithm,
while the second is an actor-critic scheme with two critics: a \textit{cost critic} and an \textit{entropy critic}. Throughout this section, we will assume that an average-cost MDP $(\mc{S}, \mc{A}, p, c)$ is fixed. The reward setting can be accommodated with minor changes by Remark \ref{rem:kappa_for_rewards}.
\subsection{Information-Directed REINFORCE}
The classic REINFORCE algorithm \cite{williams1992simple} 
generates a single, finite trajectory using a fixed policy, estimates the gradient of $J(\theta)$ based on the trajectory, and performs a corresponding stochastic gradient descent step.
We present a related algorithm, Information-Directed REINFORCE (ID-REINFORCE), that proceeds along similar lines to minimize the more complicated objective \eqref{eqn:OIR}.
%
At each time-step $t$, the algorithm generates a trajectory using the current policy $\pi_{\theta_t}$. It then forms estimates of $J(\theta_t)$ and $H(d_{\theta_t})$ and in turn uses these to estimate $\nabla \rho(\theta_t)$ by leveraging \eqref{eqn:info_ratio_gradient_thm}. This gradient estimate is then used to update the policy parameters. Note that, in order to estimate $H(d_{\theta_t})$, it is necessary to first estimate $d_{\theta_t}$. This task is addressed both implicitly and explicitly in previous works \cite{hazan2019provably, lee2019efficient, zhang2020variational}. As in \cite{hazan2019provably}, for ease of exposition we assume access to an oracle \textsc{DensityEstimator} that returns the occupancy measure $d_{\theta} = \textsc{DensityEstimator} (\theta)$ when provided with input policy parameter $\theta \in \Theta$. When $\mc{S}$ is finite and not too large, \textsc{DensityEstimator} can be implemented by computing the empirical visitation probabilities for each of the states $s \in \mc{S}$ based on sample trajectories. We focus on this setting in this paper. When $\mc{S}$ is large or continuous, on the other hand, various parametric and nonparametric density estimation techniques can be used to implement \textsc{DensityEstimator}. Pseudocode for ID-REINFORCE is given in Algorithm \ref{alg:InfoREINFORCE}.

\hspace{-7mm}
\begin{minipage}{0.47\textwidth}
    \begin{algorithm}[H] 
    \centering
    \caption{ID-REINFORCE} 
    \label{alg:InfoREINFORCE} 
    \begin{algorithmic}[1] 
        \footnotesize
        \STATE{{\textbf{Initialization:}} Rollout length $K$, step-sizes $\eta > 0$ and $\tau \in (0, 1]$, policy class $\{ \pi_{\theta} \}_{\theta \in \Theta}$, entropy constant $\kappa \geq 0$. Sample $s_0$ and $\theta_0$, select $\mu^H_{-1}$, $\mu^J_{-1} > 0$. $t \gets 0$.}
        \REPEAT
            \STATE Generate $\{ (s_i, a_i) \}_{i = 1, \ldots, K} \sim \pi_{\theta_t}$
            \STATE $\widehat{J(\theta_t)} = \frac{1}{K} \sum_{i=1}^{K} c(s_i, a_i)$
            \STATE $\mu^J_t = (1 - \tau) \mu^J_{t-1} + \tau \widehat{J(\theta_t)}$ 
            %
            %
            %
            \STATE $d_{\theta_t} = \textsc{DensityEstimator}(\theta_t)$ 
            \STATE $\widehat{H(d_{\theta_t})} = \frac{1}{K} \sum_{i=1}^K \left(- \log d_{\theta_t}(s_i) \right)$
            \STATE $\mu^H_t = (1 - \tau) \mu^H_{t-1} + \tau \widehat{H(d_{\theta_t})}$ 
            %
            %
            %
            \FOR{$i = 1, \ldots, k$}
                \STATE $\delta_i^J = c(s_i, a_i) - \mu^J_t$
                \STATE $\delta_i^H = - \log d_{\theta_t}(s_i) - \mu^H_t$
                \STATE $\psi_i = \nabla \log \pi_{\theta_t}(a_i | s_i)$
                \STATE $\Delta_i = \delta_i^J \left(\kappa + \mu^H_t \right)  - \mu^J_t \delta_i^H$
            \ENDFOR
            %
            %
            %
            %
            \STATE $\theta_{t+1} = \theta_t - \eta \frac{1}{K \left[ \kappa + \mu^H_t \right]^2} \sum_{i=1}^K \Delta_i \psi_i$
            %
            \STATE $t \gets t + 1$
        \UNTIL{convergence}
        \normalsize
    \end{algorithmic}
    \end{algorithm}
\end{minipage}
\hfill
\begin{minipage}{0.51\textwidth}
    \begin{algorithm}[H] 
    \centering
    \caption{IDAC} 
    \label{alg:irac} 
    \begin{algorithmic}[1] 
        \footnotesize
        \STATE{{\textbf{Initialization:}} Rollout length $K$, stepsizes $\{ \alpha_t \}, \{ \beta_t \}, \{ \tau_t \}$, policy class $\{ \pi_{\theta} \}_{\theta \in \Theta}$, critic class $\{ v_{\omega} \}_{\omega \in \Omega}$, entropy constant $\kappa \geq 0$. Sample $s_0, \theta_0, \omega_0^J, \omega_0^H$, select $\mu^H_{-1}$, $\mu^J_{-1} > 0$. $t \gets 0$.}
        \vspace{-3mm}
        \REPEAT
            \STATE Generate $\{ (s_i, a_i) \}_{i = 1, \ldots, K} \sim \pi_{\theta_t}$
            %
            %
            \STATE $\mu^J_t = (1 - \tau) \mu^J_{t-1} + \tau \frac{1}{K} \sum_{i=1}^{K} c(s_i, a_i)$ 
            %
            %
            \STATE $d_{\theta_t} = \textsc{DensityEstimator}(\theta_t)$ 
            %
            %
            \STATE $\mu^H_t = (1 - \tau) \mu^H_{t-1} - \tau \frac{1}{K} \sum_{i=1}^K \log d_{\theta_t}(s_i)$ 
            \FOR{$i = 1, \ldots, K$}
                \STATE Set $v_{\omega^J_t}(s_{K+1}) = v_{\omega^H_t}(s_{K+1}) = 0$
                %
                \tiny
                \STATE $\delta^J_i = c(s_i, a_i) - \mu^J_t + v_{\omega^J_t}(s_{i+1}) - v_{\omega^J_t}(s_i)$ 
                \STATE $\delta^H_i = - \log d_{\theta_t}(s_i) - \mu^H_t + v_{\omega^H_t}(s_{i+1}) - v_{\omega^H_t}(s_i)$ 
                \footnotesize
                \STATE $\psi_i = \nabla \log \pi_{\theta_t}(a_i | s_i)$
                \STATE $\Delta_i = \delta_i^J \left(\kappa + \mu^H_t \right)  - \mu^J_t \delta_i^H$
                %
                \footnotesize
            \ENDFOR
            %
            %
            \STATE $\omega^J_{t+1} = \omega^J_t + \alpha \frac{1}{K} \sum_{i=1}^K \delta^J_i \nabla v_{\omega^J_t}(s_i)$  
            %
            %
            \STATE $\omega^H_{t+1} = \omega^H_t + \alpha \frac{1}{K} \sum_{i=1}^K \delta^H_i \nabla v_{\omega^H_t}(s_i)$ 
            %
            %
            %
            %
            \STATE $\theta_{t+1} = \theta_t - \beta \frac{1}{ K \left[ \kappa + \mu^H_t \right]^2} \sum_{i=1}^K \Delta_i \psi_i$
            %
            \STATE $t \gets t + 1$
        \UNTIL{convergence}
        \normalsize
    \end{algorithmic}
    \end{algorithm}
\end{minipage}

\subsection{Information-Directed Actor-Critic}
%
%
We next present the Information-Directed Actor-Critic (IDAC) algorithm, a variant of the classic actor-critic algorithm \cite{konda02, bhatnagar2009natural} with two critics: the standard critic corresponding to average cost $J(\theta)$, and an entropy critic corresponding to the shadow MDPs $(\mc{S}, \mc{A}, p, r_t)$,
$t \geq 0$, where $r_t(s, a) = - \log d_{\theta_t}(s)$ is the shadow reward discussed in the proof of Theorem \ref{thm:entropy_gradient}. We assume access to the \textsc{DensityEstimator} oracle throughout.
%
%
The classic actor-critic algorithm for minimizing $J(\theta)$ alternates between \textit{critic} and \textit{actor} updates. At each time-step, it first computes the temporal difference (TD) error, which is a bootstrapped estimate of the amount by which the current state value function approximator, known as the \textit{critic}, over- or underestimates the true value of the current state (see \cite{sutton2018reinforcement} for details). This TD error is then used to update the critic, which is in turn used to update the policy, or \textit{actor}.
For IDAC, we modify the classic scheme by: (i) introducing an entropy critic to estimate the entropy gradient (lines 10 and 14), and (ii) altering the policy update to take a gradient descent step in the direction $-\nabla \rho(\theta_t)$ instead of $-\nabla J(\theta_t)$ (line 15).
%
%
Pseudocode is provided in Algorithm~\ref{alg:irac}.

\subsection{Density Estimation Issue} As discussed above, these algorithms are similar to the techniques described in \cite{hazan2019provably} in that they need to estimate the state density, which can be inefficient in continuous, high-dimensional spaces. There are two promising approaches for alleviating this issue. First, a variety of more sophisticated density estimation techniques have been successfully employed in RL and imitation learning in continuous settings, including kernel density estimation, variational autoencoders, energy-based models, and autoregressive models \cite{hazan2019provably, lee2019efficient, kim2021imitation}. Second, particle-based methods have recently been successfully used to avoid density estimation altogether by directly estimating occupancy measure entropy \cite{mutti2021task, yarats2021reinforcement, liu2021behavior}.
%
%
%
%
Thus, though a limitation of the present algorithms, the density estimation issue can likely be mitigated, providing an important direction for future work.

%% file: sections/Convergence_Main_Body_NEW.tex
\section{Theoretical Results} \label{sec:theoretical_results}



\newstuff{In this section we provide key results underpinning policy search for the OIR problem. In \S\ref{sec:hidden_concavity}, we show that all stationary points of $\rho(\theta)$ are in fact global minimizers. In \S\ref{sec:conv_rate}, we prove that the stochastic gradient descent scheme underlying ID-REINFORCE enjoys a non-asymptotic convergence rate depending on $\kappa$, the policy class, and ergodicity properties of the underyling MDP. Finally, \S\ref{sec:as_conv} establishes that IDAC enjoys asymptotic, almost sure (a.s.) convergence to a neighborhood of a stationary point. Taken together, these results prove that both algorithms converge to globally optimal solutions under suitable conditions.}


\subsection{Stationarity Implies Global Optimality} \label{sec:hidden_concavity}

As we will see, the OIR optimization problem enjoys a powerful \textit{hidden quasiconcavity} property: under certain conditions on the set $\Theta$ and the policy class $\{ \pi_{\theta} \}_{\theta \in \Theta}$, stationary points of $\rho(\theta)$ correspond to global optima of the OIR minimization problem
\begin{mini}
    {\theta \in \Theta}{\rho(\theta) = \frac{J(\theta)}{\kappa + H(d_{\theta})}}{}{}. \label{stat_implies_glob:1}
\end{mini}
This result is surprising, as the objective function $\rho(\theta)$ is typically highly non-convex.
Let $\Theta \subset \mathbb{R}^{k}$ be convex and let a parametrized policy class $\{ \pi_{\theta} \}_{\theta \in \Theta}$ be given. Let $\lambda : \Theta \rightarrow \mc{D}(\mc{S} \times \mc{A})$ be a function mapping each parameter vector $\theta \in \Theta$ to the state-action occupancy measure $\lambda(\theta) := \lambda_{\theta} := \lambda_{\pi_{\theta}}$ induced by the policy $\pi_{\theta}$ over $\mc{S} \times \mc{A}$. 
We make the following assumptions.
\begin{assumption}\label{assum:actor_conditions}
	The set $\Theta$ is compact. For any $s \in \mathcal{S}, a \in \mathcal{A}$, the function $\pi_\theta(a | s)$ is continuously differentiable with respect to $\theta$ on $\Theta$, and the Markov chain induced by $\pi_\theta$ on $\mathcal{S}$ is ergodic.
\end{assumption}
\begin{assumption} \label{assum:hidden_concavity_info_ratio} The following statements hold: \\
%
\textbf{1.} $\lambda(\cdot)$ is a bijection between $\Theta$ and $\lambda(\Theta)$, and $\lambda(\Theta)$ is compact and convex. \\
\textbf{2.} Let $h(\cdot) := \lambda^{-1}(\cdot)$ denote the inverse mapping of $\lambda(\cdot)$. $h(\cdot)$ is Lipschitz continuous. \\
\textbf{3.} The Jacobian matrix $\nabla \lambda (\theta)$ is Lipschitz on $\Theta$.
\end{assumption}

Assumption \ref{assum:actor_conditions} is standard in the policy gradient literature, and it implies that $\nabla \rho(\theta)$ exists, for all $\theta \in \Theta$. Assumption \ref{assum:hidden_concavity_info_ratio} holds for reasonable examples and can likely be proven to hold in the tabular setting under suitable ergodicity conditions on the underlying MDP.
\newstuff{The following is an example for which Assumption \ref{sec:hidden_concavity} holds.

\begin{example} \label{ex:hidden_concavity_example}
Consider an MDP with state space $\mc{S} = \{ s_1, s_2 \}$, action space $\mc{A} = \{ L, R \}$, a transition probability function $p$ satisfying $p(s' | s, a) > 0$, for all $s, s' \in \mc{S}, a \in \mc{A}$, and an arbitrary cost function $c : \mc{S} \times \mc{A} \rightarrow \mathbb{R}$. Given a policy $\pi$, we can represent it in tabular form as a vector $\pi \in \mathbb{R}^4$ with non-negative entries such that $\pi(L | s_i) + \pi(R | s_i) = 1$, for $i \in \{ 1, 2 \}$. The transition probability matrix of the Markov chain induced by $\pi$ on $\mc{S}$ is as follows:
%
\small
\begin{equation*} P_{\pi} =
%
\begin{bmatrix}
    p(s_1 | s_1, L) \pi( L | s_1) + p(s_1 | s_1, R) \pi( R | s_1) & p(s_2 | s_1, L) \pi(L | s_1) + p(s_1 | s_1, R) \pi(R | s_1) \\
    %
    p(s_1 | s_2, L) \pi(L | s_2) + p(s_1 | s_2, R) \pi(R | s_2) & p(s_2 | s_2, L) \pi(L | s_2) | p(s_2 | s_2, R) \pi(R | s_2)
\end{bmatrix}.
\end{equation*}
\normalsize
%
The state occupancy measure $d_{\pi}$ induced by $\pi$ can be obtained by solving the system of equations $P_{\pi}^T d = d, \bm{1}^T d = 1$ in the unknown $d = \begin{bmatrix} d_1 & d_2 \end{bmatrix}^T$. Some simple algebra yields
%
\small
\begin{align*}
    d_1 &= \frac{ p(s_1 | s_2, L) \pi(L | s_2) + p(s_1 | s_2, R) \pi(R | s_2) }{ 1 + p(s_1 | s_2, L) \pi(L | s_2) + p(s_1 | s_2, R) \pi(R | s_2) - p(s_1 | s_1, L) \pi( L | s_1) - p(s_1 | s_1, R) \pi( R | s_1)}.
\end{align*}
\normalsize
%
Notice that, since $p(s_i | s_j, a) > 0, \pi(a | s_i) \geq 0$ and $\pi(L | s_i) + \pi(R | s_i) = 1$, for all $i, j \in \{ 1, 2 \}$ and $a \in \{ L, R \}$, $d_1$ is always defined and strictly positive. Since $d_2 = 1 - d_1$, the same statement holds for $d_2$. The state occupancy measure induced by $\pi$ is thus given by the vector $d_{\pi} = d = \begin{bmatrix} d_1 & d_2 \end{bmatrix}^T$.

%
We now show that, for the MDP $(\mc{S}, \mc{A}, p, c)$ just specified, Assumption \ref{assum:hidden_concavity_info_ratio} holds.
%
Let $\Theta \subset \mathbb{R}^4$ denote the set of vector representations of all valid policies $\pi$, and let $\lambda : \Theta \rightarrow \mc{D}(\mc{S} \times \mc{A})$ denote the function mapping policies $\pi$ to state-action occupancy measures $\lambda(\pi)$. Clearly $\lambda(\pi)(s, a) = d_{\pi}(s) \pi(a | s)$, for all $s \in \mc{S}, a \in \mc{A}$.
%
It is known that part one of Assumption \ref{assum:hidden_concavity_info_ratio} holds (see, e.g., \cite{Puterman_2014}). It can furthermore be shown (see the proof of Prop.~H.2 in \cite{zhang2020variational}) that part two holds. All that remains is to show that part three holds by proving that $\nabla \lambda(\pi)$ is Lipschitz in $\pi$. Notice that the identity function is Lipschitz and trivially bounded over any bounded domain, and recall that the product of two Lipschitz and bounded functions is a Lipschitz function. Since the domain $\Theta$ is bounded, this means that, if $d_{\pi}$ is Lipschitz and bounded in $\pi$, it will follow that $\lambda(\pi)$ is Lipschitz in $\pi$ over $\Theta$, since it is a product of Lipschitz, bounded functions. But the partial derivatives of $d_1$ with respect to $\pi(L | s_1), \pi(R | s_1), \pi(L | s_2),$ and $\pi(R | s_2)$ are all continuous and bounded for all valid policy vectors $\pi$, so $d_1$ is Lipschitz in $\pi$. Since $d_2 = 1 - d_1$, this means that $d_{\pi} = \begin{bmatrix} d_1 & d_2 \end{bmatrix}^T$ is Lipschitz in $\pi$. Finally, we already know that $d_{\pi}$ is bounded, so $\lambda(\pi)$ is therefore Lipschitz in $\pi$, and thus part three of Assumption~\ref{assum:hidden_concavity_info_ratio}~holds.

\end{example}
}

We now have the following theorem. The key idea behind the proof is to show that the stationary point $\theta^*$ corresponds to an optimal solution to the concave program \eqref{opt:Q3} and thus also provides an optimal solution to the quasiconvex OIR minimization problem \eqref{opt:Q0}. The proof builds on that of Theorem~4.2 of \cite{zhang2020variational}, with key modifications to accommodate the fact that the underlying OIR optimization problem is not convex, but \textit{quasiconvex} in the state-action occupancy measure. \newstuff{In particular, the result in \cite{zhang2020variational} holds for concave functionals of the state-action occupancy measure only, not ones involving \emph{quasiconcave} functionals. The critical innovation in the proof of Theorem \ref{thm:stationarity_implies_global_optimality_info_ratio} below is to leverage properties of the perspective transform combined with the smoothness conditions of Assumption 5.2 to extend the hidden concavity analysis of \cite{zhang2020variational} to the quasiconcave setting.}
%
\begin{theorem} \label{thm:stationarity_implies_global_optimality_info_ratio}
    Let Assumptions \ref{assum:actor_conditions} and \ref{assum:hidden_concavity_info_ratio} hold. Let $\theta^*$ be a stationary point of \eqref{stat_implies_glob:1}, i.e.,
    %
    $\nabla \rho(\theta^*) = 0.$
    Then $\theta^*$ is globally optimal for \eqref{stat_implies_glob:1}.
\end{theorem}
\begin{proof}

%
%
\newstuff{
We reformulate \eqref{stat_implies_glob:1} as a maximization problem. Let $q(\theta) = 1 / \rho(\theta) = (\kappa + H(d_{\theta})) / J(\theta)$. Let $\widehat{H}(\lambda_{\theta}) = H(d_{\theta})$, where $\widehat{H}(\lambda) = H(d^{\lambda})$ is the entropy of the state occupancy measure $d^{\lambda} \in \mc{D}(\mathcal{S})$ given by $d^{\lambda}_s = \sum_a \lambda_{sa}$, and recall that $J(\theta) = c^T \lambda_{\theta}$, for some vector $c > 0$ of costs. This means that $q(\theta) = (\kappa + \widehat{H}(\lambda_{\theta})) / J(\lambda_{\theta})$. In what follows we prove that $\theta^*$ is globally optimal for $\max_{\theta \in \Theta} q(\theta)$. By Lemma \ref{lemma:Q0_eq_Q1}, this will imply that $\theta^*$ is globally optimal for $\min_{\theta \in \Theta} \rho(\theta)$.
Also note that, since $\rho(\theta)$ is strictly positive on $\Theta$, we know $q(\theta)$ is differentiable in $\theta$ and $\nabla q(\theta) = - \nabla \rho(\theta) / [\rho(\theta)]^2$, for all $\theta \in \Theta$. Since $\nabla \rho(\theta^*) = 0$ by assumption, this means $\nabla q(\theta^*) = 0$, so $\theta^*$ is a stationary point of the optimization problem $\max_{\theta \in \Theta} q(\theta)$.

We now transform the problem $\max_{\theta \in \Theta} q(\theta)$ to a concave program. For $z \in \mathbb{R}^{|\mc{S}| |\mc{A}| + 1}$, let $y \in \mathbb{R}^{|\mc{S}| |\mc{A}|}$ denote all but the last entry in $z$, and let the scalar $t$ denote the last entry of $z$. We will write $z = (y, t)$ for brevity. Let $\zeta : \mc{D}(\mc{S} \times \mc{A}) \rightarrow \mathbb{R}^{|\mc{S}| |\mc{A}| + 1}$ be the mapping given by $\zeta(\lambda) = (\lambda / J(\lambda), 1 / J(\lambda))$. Consider the optimization problems
}

\begin{minipage}{0.43\textwidth}
\begin{maxi}
    {\lambda \in \lambda(\Theta)}{\frac{\kappa + \widehat{H}(\lambda)}{J(\lambda)}}{}{},  \label{stat_implies_glob:2}
\end{maxi}
\end{minipage}
\hfill
\begin{minipage}{0.43\textwidth}
\begin{maxi}
    {z \in (\zeta \circ \lambda)(\Theta)}{ P_{\kappa, \widehat{H}}(z)}{}{}, \label{stat_implies_glob:3}
\end{maxi}
\end{minipage}
\vspace{2mm}

\noindent where $P_{\kappa, \widehat{H}} : \mathbb{R}^{|\mc{S}| |\mc{A}| + 1} \rightarrow \mathbb{R}$ denotes the perspective transformation of $\kappa + \widehat{H}(\lambda)$, given by $P_{\kappa + \widehat{H}} (z) = P_{\kappa, \widehat{H}} ((y, t)) = t ( \kappa + \widehat{H}(y / t))$. For notational convenience we henceforth drop the dependency on $\kappa$ and simply write $P_{\widehat{H}}$ and $\widehat{H}$ instead of $P_{\kappa, \widehat{H}}$ and $\kappa + \widehat{H}$. Recall that, since $\widehat{H}$ is concave over the region $\mc{D}(\mc{S} \times \mc{A})$, its perspective transform $P_{\widehat{H}}$ is concave over the region $\zeta(\mc{D}(\mc{S} \times \mc{A}))$. $P_{\widehat{H}}$ is thus concave over the convex, compact region $(\zeta \circ \lambda)(\Theta) \subseteq \zeta(\mc{D}(\mc{S} \times \mc{A}))$.


\newstuff{
The remainder of the proof provides the technical details demonstrating that $z^*$ is a stationary point of \eqref{stat_implies_glob:3}. Since \eqref{stat_implies_glob:3} is a concave program, this will imply that $z^*$ and thus $\theta^*$ are globally optimal for their respective problems.
%
%
We first show that the conditions of Assumption \ref{assum:hidden_concavity_info_ratio} can be extended to the mapping $\zeta \circ \lambda$. To do so, we need to prove:

\begin{center}
\begin{varwidth}{0.6\textwidth}
\begin{enumerate}[(i)]
    \item $\zeta \circ \lambda$ gives a bijection between $\Theta$ and $(\zeta \circ \lambda)(\Theta)$;
    \item $\zeta \circ \lambda$ has a Lipschitz continuous inverse; and
    \item the Jacobian $\nabla_{\theta}(\zeta \circ \lambda)(\theta)$ is Lipschitz.
\end{enumerate}
\end{varwidth}
\end{center}
\vspace{2mm}

To prove (i), recall $\zeta(\lambda) = (\lambda / J(\lambda), 1 / J(\lambda))$. We know $\zeta$ is a surjection onto $(\zeta \circ \lambda)(\Theta)$ by definition, so we just need to show it is injective. Fix $\lambda \neq \lambda'$. If $J(\lambda) = J(\lambda')$, then $\lambda / J(\lambda) \neq \lambda' / J(\lambda')$, so $\zeta(\lambda) \neq \zeta(\lambda')$. If $J(\lambda) \neq J(\lambda')$, on the other hand, then $1 / J(\lambda) \neq 1 / J(\lambda')$, so again $\zeta(\lambda) \neq \zeta(\lambda')$. Therefore $\zeta$ is injective and thus gives a bijection. Combined with Assumption \ref{assum:hidden_concavity_info_ratio}, the foregoing implies that $\zeta \circ \lambda$ gives a bijection between $\Theta$ and $(\zeta \circ \lambda)(\Theta)$, proving (i).

For (ii), the inverse of $\zeta$ is clearly $\zeta^{-1}(z) = \zeta^{-1}((y, t)) = y / t$. Since $0 < \min_i c_i \leq t \leq \max_i c_i < \infty$, $\zeta^{-1}$ has continuous, bounded partial derivatives and is thus Lipschitz continuous on $(\zeta \circ \lambda)(\Theta)$. Since the composition of Lipschitz functions is Lipschitz, $k = (\zeta \circ \lambda)^{-1} = \lambda^{-1} \circ \zeta^{-1}$ is Lipschitz continuous, proving (ii).

For (iii), an application of the chain rule gives $\nabla_{\theta}(\zeta \circ \lambda)(\theta) = \left[ \nabla_{\lambda} \zeta( \lambda(\theta) ) \right]^T \nabla_{\theta} \lambda(\theta)$. Clearly $\nabla_{\lambda} \zeta( \lambda )$ is Lipschitz continuous and bounded over the compact set $\Theta$. Since $\nabla_{\theta} \lambda(\theta)$ is (Lipschitz) continuous and bounded over $\Theta$, we know $\lambda(\theta)$ is Lipschitz, implying that $\nabla_{\lambda} \zeta( \lambda(\theta) )$ is Lipschitz and bounded on $\theta \in \Theta$. Furthermore, $\nabla_{\theta} \lambda(\theta)$ is Lipschitz by assumption and bounded over $\Theta$, so all entries in the matrix product $\left[ \nabla_{\lambda} \zeta( \lambda(\theta) ) \right]^T \nabla_{\theta} \lambda(\theta)$ are sums and products of Lipschitz, bounded functions over $\Theta$. This implies that $\nabla_{\theta}(\zeta \circ \lambda)(\theta)$ is Lipschitz on $\Theta$, proving (iii).
}

We now move on to the bounding arguments that will ultimately prove that $z^*$ is a stationary point of \eqref{stat_implies_glob:3}. First, notice that
$$(P_{\widehat{H}} \circ \zeta \circ \lambda)(\theta) = P_{\widehat{H}}(\zeta(\lambda(\theta))) = \frac{\kappa + \widehat{H}(\lambda(\theta))}{J(\lambda(\theta))} = \frac{\kappa + \widehat{H}(\lambda_{\theta})}{J(\lambda_{\theta})} = q(\theta),$$
so $\nabla_{\theta} (P_{\widehat{H}} \circ \zeta \circ \lambda)(\theta^*) = \nabla_{\theta} \frac{\kappa + \widehat{H}(\lambda_{\theta^*})}{J(\lambda_{\theta^*})} = \nabla q(\theta^*) = 0$. Since $P_{\widehat{H}}$ is concave and locally Lipschitz on $(\zeta \circ \lambda)(\Theta)$, by the chain rule we have
$\nabla_{\theta} (P_{\widehat{H}} \circ \zeta \circ \lambda)(\theta^*) = \left[ \nabla_{\theta} (\zeta \circ \lambda)(\theta^*) \right]^T \nabla_z P_{\widehat{H}}(z^*) = 0,$
where $z^* = (\zeta \circ \lambda)(\theta^*)$.
This trivially implies that, for all $\theta \in \Theta$,
\begin{align} \label{stat_implies_glob:4}
    \langle \nabla_z P_{\widehat{H}}(z^*), \nabla_{\theta} (\zeta \circ \lambda)(\theta^*) (\theta - \theta^*) \rangle = \langle \left[ \nabla_{\theta} (\zeta \circ \lambda)(\theta^*) \right]^T \nabla_z P_{\widehat{H}}(z^*), \theta - \theta^* \rangle = 0.
\end{align}
Equation \eqref{stat_implies_glob:4} is important to the bounding arguments presented next.

In the following equations, let $\theta = k(z)$ and $\theta^* = k(z^*)$. Adding and subtracting $\langle \nabla_z P_{\widehat{H}}(z^*), \nabla_{\theta} (\zeta \circ \lambda)(\theta^*) (\theta - \theta^*)$, using equation \eqref{stat_implies_glob:4}, and applying the Cauchy-Schwarz inequality, we get
\begin{align*}
    \langle \nabla_z & P_{\widehat{H}}(z^*), z - z^* \rangle = \langle \nabla_z P_{\widehat{H}}(z^*), (\zeta \circ \lambda)(\theta) - (\zeta \circ \lambda)(\theta^*) \rangle \\
    &= \langle \nabla_z P_{\widehat{H}}(z^*), \nabla_{\theta} (\zeta \circ \lambda)(\theta^*) (\theta - \theta^*) \rangle \nonumber \\
    & \hspace{1cm} + \langle \nabla_z P_{\widehat{H}}(z^*), (\zeta \circ \lambda)(\theta) - (\zeta \circ \lambda)(\theta^*) - \nabla_{\theta} (\zeta \circ \lambda)(\theta^*) (\theta - \theta^*) \rangle \\
    &= \langle \nabla_z P_{\widehat{H}}(z^*), (\zeta \circ \lambda)(\theta) - (\zeta \circ \lambda)(\theta^*) - \nabla_{\theta} (\zeta \circ \lambda)(\theta^*) (\theta - \theta^*) \rangle \\
    &\leq \norm{\nabla_z P_{\widehat{H}}(z^*)} \norm{(\zeta \circ \lambda)(\theta) - (\zeta \circ \lambda)(\theta^*) - \nabla_{\theta} (\zeta \circ \lambda)(\theta^*) (\theta - \theta^*) }.
\end{align*}
Since $\nabla_{\theta} (\zeta \circ \lambda)(\theta)$ is Lipschitz, there exists $K_0 > 0$ such that
\begin{align*}
    \norm{(\zeta \circ \lambda)(\theta) - (\zeta \circ \lambda)(\theta^*) - \nabla_{\theta} (\zeta \circ \lambda)(\theta^*) (\theta - \theta^*)} \leq \frac{K_0}{2} \norm{\theta - \theta^*}^2.
\end{align*}
In addition, $k = (\zeta \circ \lambda)^{-1}$ is Lipschitz, so there exists $K_1 > 0$ such that
\begin{align*}
    \norm{\theta - \theta^*}^2 = \norm{k(z) - k(z^*)}^2 \leq K_1^2 \norm{z - z^*}^2.
\end{align*}
Combining these inequalities yields that
\begin{align*}
    \langle \nabla_z P_{\widehat{H}}(z^*), z - z^* \rangle \leq \frac{K_0 K_1^2}{2} \norm{\nabla_z P_{\widehat{H}}(z^*)} \norm{z - z^*}^2,
\end{align*}
for all $z \in (\zeta \circ \lambda)(\Theta)$. Since $(\zeta \circ \lambda)(\Theta)$ is convex, we can replace $z$ above with $(1 - \alpha) z^* + \alpha z$ for any $\alpha \in [0, 1]$, which gives
\begin{align*}
    \alpha \langle \nabla_z P_{\widehat{H}}(z^*), z - z^* \rangle \leq \frac{K_0 K_1^2 \alpha^2}{2} \norm{\nabla_z P_{\widehat{H}}(z^*)} \norm{z - z^*}^2,
\end{align*}
for all $z \in (\zeta \circ \lambda)(\Theta)$ and $\alpha \in [0, 1].$
Dividing both sides by $\alpha$ and taking the limit as $\alpha$ approaches 0 from above, we obtain
\begin{align*}
    \langle \nabla_z P_{\widehat{H}}(z^*), z - z^* \rangle \leq 0, \text{ for all } z \in (\zeta \circ \lambda)(\Theta).
\end{align*}
Since problem \eqref{stat_implies_glob:3} is concave in $z$, this implies that $z^* = (\zeta \circ \lambda)(\theta^*)$ is a stationary point of that problem. The solution $z^*$ is therefore a global optimal solution to \eqref{stat_implies_glob:3}, \newstuff{implying that $\theta^*$ is globally optimal for \eqref{stat_implies_glob:1}.}
\end{proof}

This powerful hidden quasiconcavity property implies that any policy gradient algorithm that can be shown to converge to a stationary point of the OIR optimization problem $\min_{\theta \in \Theta} \rho(\theta)$ in fact converges to a global optimum. This greatly strengthens the convergence results provided next by guaranteeing that they apply to \textit{global} optima.
In contrast to the global optimality guarantees for tabular, softmax policy search established in \cite{bhandari2019global, Agarwal_Kakade_Lee_Mahajan_2020, mei2020global, zhang2020variational, bedi2021sample} using persistent exploration conditions, our result instead builds on hidden concavity arguments from \cite{zhang2020variational}, which apply to parameterized policies. However, Theorem \ref{thm:stationarity_implies_global_optimality_info_ratio} generalizes these results in important ways. First, it applies to ratio objectives, which have not been addressed in prior work. In addition, we establish hidden \textit{quasiconcavity} for ratio objectives, not hidden \emph{concavity}, which requires reformulation via an application of the perspective transform [cf. \S\ref{sec:concave_program_2}]. In these ways, Theorem  \ref{thm:stationarity_implies_global_optimality_info_ratio} is a strict generalization of existing results for the landscape of RL objectives.


\subsection{Non-Asymptotic Convergence Rate} \label{sec:conv_rate}

We next establish a convergence rate for the following projected gradient descent scheme for solving the OIR minimization problem \eqref{stat_implies_glob:1}:
%
%
\begin{equation} \label{eqn:prox_update_rho}
    \small
    \theta_{t+1} =  \text{Proj}_{\Theta} \left( \theta_t - \eta \nabla \rho(\theta_t) \right) = \argmin_{\theta} [\rho(\theta_t) + \left< \nabla \rho(\theta_t), \theta - \theta_t \right> + \frac{1}{2 \eta} \norm{\theta - \theta_t}^2],
    \normalsize
\end{equation}
for a fixed stepsize $\eta > 0$, where $\text{Proj}_{\Theta}$ denotes euclidean projection onto $\Theta$ and the second equality holds by the convexity of $\Theta$. Note that \eqref{eqn:prox_update_rho} is a reformulation of Algorithm \ref{alg:InfoREINFORCE} with null gradient estimation error and projection onto the set $\Theta$; we assume the projection operation for the purposes of analysis, and we discuss the gradient estimation issue at the end of this subsection. 

%
Let $\Theta \subset \mathbb{R}^{k}$, $\{ \pi_{\theta} \}_{\theta \in \Theta}$, and $\lambda : \Theta \rightarrow \mc{D}(\mc{S} \times \mc{A})$ be as in the previous section.
Recall the mapping $\zeta : \mc{D}(\mc{S} \times \mc{A}) \rightarrow \mathbb{R}^{|\mc{S}| |\mc{A}| + 1}$ from the proof of Theorem \ref{thm:stationarity_implies_global_optimality_info_ratio}, which was defined to be $\zeta(\lambda) = (\lambda / c^\top \lambda, 1 / c^\top \lambda)$, where $c \in \mathbb{R}^{|\mc{S}| |\mc{A}|}, c > 0$ is a vector of positive costs. Notice that, under the ergodicity conditions in Assumption \ref{assum:actor_conditions} and properties of entropy, $\min_{\theta} \rho(\theta) > 0$ and $\max_{\theta} \rho(\theta) < \infty$. In addition to Assumptions \ref{assum:actor_conditions} and \ref{assum:hidden_concavity_info_ratio}, we will need the following.
\begin{assumption} \label{assum:OIR_lipschitz_jacobian}
    $\nabla \rho(\theta)$ is Lipschitz and $L > 0$ is the smallest number such that $\norm{\nabla \rho(\theta) - \nabla \rho(\theta') } \leq L \norm{\theta - \theta'},$ for all $\theta, \theta' \in \Theta$.
\end{assumption}
We have the following convergence rate result for the projected gradient descent scheme \eqref{eqn:prox_update_rho}. \newstuff{The key idea behind the proof is to link the objective function $\rho(\theta)$ that the updates \eqref{eqn:prox_update_rho} are minimizing to the concave structure of the transformed problem \eqref{opt:Q3}. This allows us to derive the bound \eqref{ineq:OIR_prox_bound} by studying related bounds for the concave objective function of \eqref{opt:Q3}.
Similar to Theorem \ref{thm:stationarity_implies_global_optimality_info_ratio} above, our key innovation is that we leverage properties of the perspective transform combined with the Lipschitz condition of Assumption 5.4 to extend the rate analysis of Theorem 4.4 in \cite{zhang2020variational} -- which holds only for concave functionals of the state-action occupancy measure -- to the quasiconcave setting.}
\begin{theorem} \label{thm:OIR_convergence_rate}
    Let Assumptions 
    \ref{assum:actor_conditions},
    \ref{assum:hidden_concavity_info_ratio},
    and \ref{assum:OIR_lipschitz_jacobian} hold. Denote the diameter of the convex, compact set $(\zeta \circ \lambda)(\Theta)$ by $D_{\zeta} = \max_{z, z' \in (\zeta \circ \lambda)(\Theta)} \norm{z - z'}$. Define $M = \max_{\theta \in \Theta} \rho(\theta)$, $m = \min_{\theta \in \Theta} \rho(\theta)$, $K = \max \{ m^2 L, M^2 m^2 L \}$, and $L_1 = \max \{ L, M^2 L \}$. Then, with $\eta = 1 / K$, for all $t \geq 0$,
    \begin{equation} \label{ineq:OIR_prox_bound}
        \rho(\theta_t) - \rho(\theta^*) \leq \frac{ 4 M^2 L_1 \ell^2 D^2_{\zeta}}{t + 1},
    \end{equation}
    \newstuff{where $\ell$ is the minimal Lipschitz constant of the inverse mapping $(\zeta \circ \lambda)^{-1}$.}
\end{theorem}

\begin{proof}
%

To draw the connection with \eqref{opt:Q3}, we first transform \eqref{eqn:prox_update_rho} into an equivalent projected gradient \textit{ascent} scheme. Define $q(\theta) = 1 / \rho(\theta) = (\kappa + H(d_{\theta})) / J(\theta)$, and notice that $\nabla q(\theta) = - \nabla \rho(\theta) / \left[ \rho(\theta) \right]^2$. The projected gradient ascent scheme becomes
\small
\begin{align}
    \theta_{t+1} &= \text{Proj}_{\Theta} \left( \theta_t - \eta \nabla \rho(\theta_t) \right) \nonumber \\
    &\labelrel={eqn:prox_transform1} \text{Proj}_{\Theta} \left( \theta_t + \eta \left[ \rho(\theta_t) \right]^2 \nabla q(\theta_t) \right) \nonumber \\
    &\labelrel={eqn:prox_transform2} \argmax_{\theta} \left( q(\theta_t) + \left< \nabla q(\theta_t), \theta - \theta_t \right> - \frac{\left[ \rho(\theta_t) \right]^2}{2 \eta} \norm{\theta - \theta_t}^2 \right). \nonumber \\
    &\labelrel={eqn:prox_transform3} \argmax_{\theta} \left( q(\theta_t) + \left< \nabla q(\theta_t), \theta - \theta_t \right> - \frac{1}{2 \eta_t} \norm{\theta - \theta_t}^2 \right), \label{eqn:OIR_prox_ascent}
\end{align}
\normalsize
where \eqref{eqn:prox_transform1} follows by noticing that $\nabla \rho(\theta) = - [\rho(\theta)]^2 \nabla q(\theta)$ and making the appropriate substitution, \eqref{eqn:prox_transform2} holds by definition of the $\text{Proj}$ operator, and \eqref{eqn:prox_transform3} results from defining $\eta_t = \eta [\rho(\theta_t)]^2 = [\rho(\theta_t)]^2 / K_0$.

We next identify a family $\{ K_c \ | \ c \in [m, M] \}$ of Lipschitz constants of the gradient $\nabla q(\theta)$. By Assumption \ref{assum:OIR_lipschitz_jacobian}, $\nabla q(\theta)$ is Lipschitz and $\norm{ \nabla q(\theta) - \nabla q(\theta')} \leq L_0 \norm{\theta - \theta'}$, for all $\theta, \theta' \in \Theta$, where $L_0 = m^2 L$. Let $K = \max \{ L_0, M^2 L_0 \} = \max \{ m^2 L, M^2 m^2 L \}$. Then, for all scalars $c \in [m , M]$, the gradient $\nabla q(\theta)$ satisfies $\norm{ \nabla q(\theta) - \nabla q(\theta')} \leq K_c \norm{\theta - \theta'},$ where $K_c = K / c^2$ is the desired Lipschitz constant. This family of Lipschitz constants will be critical in the analysis to follow.

We now study the sequence $\alpha_t = \left[ q(\theta^*) - q(\theta_t) \right] / 2 K_m \ell^2 D^2_{\zeta}$, for $t \geq 0$, ultimately using properties of the sequence to show that $q(\theta^*) - q(\theta_t) \leq 4 K_m \ell^2 D^2_{\zeta} / (1 + t)$, for all $t \geq 0$.
The remainder of the proof proceeds along lines similar to that of Theorem~4.4 in \cite{zhang2020variational}, \newstuff{with critical modifications to accommodate the use of the perspective transform, the variable transformation $\zeta$, and the non-constant stepsizes $\eta_t = \eta [ \rho(\theta_t)]^2$.}

\newstuff{
Define $\widehat{H}(\lambda) = H(d^{\lambda})$, where $H(d^{\lambda})$ is the entropy of the state occupancy measure $d^{\lambda}(s) = \sum_a \lambda(s, a)$ corresponding to the state-action occupancy measure $\lambda \in \mc{D}(\mc{S} \times \mc{A})$. For a given $\kappa \geq 0$, let $P_{\kappa, \widehat{H}} : \mathbb{R}^{|\mc{S}| |\mc{A}| + 1} \rightarrow \mathbb{R}$ denote the perspective transformation of $\kappa + \widehat{H}(\lambda)$, given by $P_{\kappa + \widehat{H}} (z) = P_{\kappa, \widehat{H}} ((y, t)) = t ( \kappa + \widehat{H}(y / t))$. For notational convenience we henceforth drop the dependency on $\kappa$ and simply write $P_{\widehat{H}}$.

Our next step is to make use of the concave structure of $P_{\widehat{H}}$, $\zeta \circ \lambda$, and \eqref{opt:Q3} to analyze $\{ \alpha_t \}_{t \in \mathbb{N}}$. We first derive several useful inequalities regarding $P_{\widehat{H}}$ and $\zeta \circ \lambda$. Notice that $q(\theta) = (\kappa + \widehat{H}(\lambda_{\theta})) / J(\lambda_{\theta}) = P_{\widehat{H}}(\zeta(\lambda_{\theta})) = P_{\widehat{H}}((\zeta \circ \lambda)(\theta))$, for all $\theta \in \Theta$. This means that we can rewrite $\alpha_t$ as}
\begin{equation} \label{eqn:alpha_redef}
    \alpha_t = \left[ P_{\widehat{H}}((\zeta \circ \lambda)(\theta^*)) - P_{\widehat{H}}((\zeta \circ \lambda)(\theta_t)) \right] / 2 K_m \ell^2 D^2_{\zeta}.
\end{equation}
Notice $P_{\widehat{H}}$ is concave over $\zeta(\mc{D}(\mc{S} \times \mc{A}))$, since $\widehat{H}$ is concave over $\mc{D}(\mc{S} \times \mc{A})$ and the perspective transform preserves concavity. $P_{\widehat{H}}$ is thus concave over the convex, compact region $(\zeta \circ \lambda)(\Theta) \subseteq \zeta(\mc{D}(\mc{S} \times \mc{A}))$.
Furthermore, since $P_{\widehat{H}}((\zeta \circ \lambda)(\theta)) = q(\theta)$, we have $\nabla P_{\widehat{H}}((\zeta \circ \lambda)(\theta)) = \nabla q(\theta)$, so $\nabla P_{\widehat{H}}((\zeta \circ \lambda)(\theta))$ is $K_c$-Lipschitz, for any $K_c = K / c^2, c \in [m, M]$. This implies (see Lem.~1.2.3 in \cite{nesterov2003introductory}), for any $c \in [m, M]$, that
\begin{equation*}
    \Big\rvert P_{\widehat{H}} ((\zeta \circ \lambda)(\theta)) - P_{\widehat{H}}((\zeta \circ \lambda)(\theta_t)) - \left< \nabla P_{\widehat{H}}((\zeta \circ \lambda)(\theta_t)), \theta - \theta_t \right> \Big\rvert \leq \frac{K_c}{2} \norm{\theta - \theta_t}^2,
\end{equation*}
whence, for any $\theta \in \Theta$,
\begin{align}
    P_{\widehat{H}} ((\zeta \circ \lambda)(\theta)) &\geq P_{\widehat{H}} ((\zeta \circ \lambda)(\theta_t)) + \left< \nabla P_{\widehat{H}} ((\zeta \circ \lambda)(\theta_t)), \theta - \theta_t \right> - \frac{K_c}{2} \norm{\theta - \theta_t}^2 \nonumber \\
    &\geq P_{\widehat{H}} ((\zeta \circ \lambda)(\theta)) - K_c \norm{\theta - \theta_t}^2. \label{ineq:conv1_0} 
\end{align}
In light of these inequalities, and using the fact that $\eta_t = [\rho(\theta_t)]^2 / K = 1 / K_{\rho(\theta_t)}$ by setting $c = \rho(\theta_t)$ in the definition of $K_c$, we have
\small
\begin{align}
    &P_{\widehat{H}} ((\zeta \circ \lambda) (\theta_{t+1})) \geq P_{\widehat{H}} ((\zeta \circ \lambda)(\theta_t)) + \left< \nabla P_{\widehat{H}} ((\zeta \circ \lambda)(\theta_t)), \theta_{t+1} - \theta_t \right> - \frac{K_{\rho(\theta_t)}}{2} \norm{\theta_{t+1} - \theta_t}^2 \nonumber \\
    &\labelrel={eqn:conv1_1} \max_{\theta \in \Theta} \left( P_{\widehat{H}} ((\zeta \circ \lambda)(\theta_t)) + \left< \nabla P_{\widehat{H}} ((\zeta \circ \lambda)(\theta_t)), \theta - \theta_t \right> - \frac{K_{\rho(\theta_t)}}{2} \norm{\theta - \theta_t}^2 \right) \nonumber \\
    &\labelrel\geq{ineq:conv1_2} \max_{\theta \in \Theta} \left( P_{\widehat{H}} ((\zeta \circ \lambda)(\theta)) - K_{\rho(\theta_t)} \norm{\theta - \theta_t}^2 \right) \nonumber \\
    &\labelrel\geq{ineq:conv1_3} \max_{\theta \in \Theta} \left( P_{\widehat{H}} ((\zeta \circ \lambda)(\theta)) - K_m \norm{\theta - \theta_t}^2 \right) \nonumber \\
    &\labelrel\geq{ineq:conv1_4} \max_{\alpha \in [0, 1]} \{ P_{\widehat{H}} ((\zeta \circ \lambda)(\theta_{\alpha})) - K_m \norm{\theta_{\alpha} - \theta_t}^2 \ \big\rvert \ \theta_{\alpha} = k \big( \alpha (\zeta \circ \lambda)(\theta^*) + (1 - \alpha) (\zeta \circ \lambda)(\theta_t) \big) \}, \label{ineq:OIR_prox_conv_rate1}
\end{align}
\normalsize
where \eqref{eqn:conv1_1} follows from the optimality of the update \eqref{eqn:OIR_prox_ascent}, \eqref{ineq:conv1_2} holds by \eqref{ineq:conv1_0}, \eqref{ineq:conv1_3} follows from the fact that $K_{\rho(\theta_t)} \leq K_m$, the \eqref{ineq:conv1_4} follows by the convexity of $(\zeta \circ \lambda)(\Theta)$.

\newstuff{Let $k(\cdot) = (\zeta \circ \lambda)^{-1}(\cdot)$ as in Theorem \ref{thm:stationarity_implies_global_optimality_info_ratio}. By Assumption \ref{assum:hidden_concavity_info_ratio} and the proof of Theorem \ref{thm:stationarity_implies_global_optimality_info_ratio}, we know that $k(\cdot)$ is $\ell$-Lipschitz.} Now notice that
\begin{align}
    P_{\widehat{H}}((\zeta \circ \lambda)(\theta_{\alpha})) &= P_{\widehat{H}} \Big( (\zeta \circ \lambda) \big( k(\alpha (\zeta \circ \lambda)(\theta^*) + (1 - \alpha) (\zeta \circ \lambda)(\theta_t)) \big) \Big) \nonumber \\
    &= P_{\widehat{H}}( \alpha (\zeta \circ \lambda)(\theta^*) + (1 - \alpha) (\zeta \circ \lambda)(\theta_t)) \nonumber \\
    &\geq \alpha P_{\widehat{H}}((\zeta \circ \lambda)(\theta^*)) + (1 - \alpha) P_{\widehat{H}}((\zeta \circ \lambda)(\theta_t)), \label{ineq:OIR_prox_conv_rate1_2}
\end{align}
where the first equality holds by the definition of $\theta_{\alpha}$ given in \eqref{ineq:OIR_prox_conv_rate1}, the second follows from the fact that $k((\zeta \circ \lambda)(\theta)) = \theta$, for any $\theta \in \Theta$, and the final inequality is yielded by the concavity of $P_{\widehat{H}}$ over $(\zeta \circ \lambda)(\Theta)$.
Furthermore,
%
\small
\begin{align}
    \norm{ \theta_{\alpha} - \theta_t }^2 &\labelrel={eqn:conv1_5} \norm{ k(\alpha (\zeta \circ \lambda)(\theta^*) + (1 - \alpha)(\zeta \circ \lambda)(\theta_t)) - k((\zeta \circ \lambda)(\theta_t)) }^2 \nonumber \\
    &\labelrel\leq{ineq:conv1_6} \ell^2 \norm{ \alpha (\zeta \circ \lambda)(\theta^*) + (1 - \alpha)(\zeta \circ \lambda)(\theta_t) - (\zeta \circ \lambda)(\theta_t) }^2 \nonumber \\
    &\leq \alpha^2 \ell^2 \norm{(\zeta \circ \lambda)(\theta^*) - (\zeta \circ \lambda)(\theta_t)}^2 \nonumber \\
    &\labelrel\leq{ineq:conv1_7} \alpha^2 \ell^2 D^2_{\zeta}, \label{ineq:OIR_prox_conv_rate1_3}
\end{align}
\normalsize
%
where \eqref{eqn:conv1_5} holds by the definition of $\theta_{\alpha}$ and the fact that $k((\zeta \circ \lambda)(\theta)) = \theta$, \eqref{ineq:conv1_6} follows since $k$ is $\ell$-Lipschitz, and \eqref{ineq:conv1_7} results from the definition of $D_{\zeta}$ given in the statement of the theorem. Now, the inequalities \eqref{ineq:OIR_prox_conv_rate1}, \eqref{ineq:OIR_prox_conv_rate1_2}, and \eqref{ineq:OIR_prox_conv_rate1_3} combine to yield
%
\small
\begin{align}
    P_{\widehat{H}} (( & \zeta \circ \lambda)(\theta^*)) - P_{\widehat{H}} ( (\zeta \circ \lambda)(\theta_{t+1})) \nonumber \\
    &\labelrel\leq{ineq:conv1_8} \min_{\alpha \in [0, 1]} \{ P_{\widehat{H}} ((\zeta \circ \lambda)(\theta^*)) - P_{\widehat{H}} ((\zeta \circ \lambda)(\theta_{\alpha})) + K_m \norm{ \theta_{\alpha} - \theta_t }^2 \nonumber \\ 
    & \hspace{2cm} \big\rvert \ \theta_{\alpha} = k \big( \alpha (\zeta \circ \lambda)(\theta^*) + (1 - \alpha) (\zeta \circ \lambda)(\theta_t) \big) \} \nonumber \\
    &\labelrel\leq{ineq:conv1_9} \min_{\alpha \in [0, 1]} \Big( P_{\widehat{H}} ((\zeta \circ \lambda)(\theta^*)) - \alpha P_{\widehat{H}}((\zeta \circ \lambda)(\theta^*)) \nonumber \\
    & \hspace{2cm} - (1 - \alpha) P_{\widehat{H}} ((\zeta \circ \lambda)(\theta_t)) + K_m \alpha^2 \ell^2 D^2_{\zeta} \Big) \nonumber \\
    &= \min_{\alpha \in [0, 1]} \Big( (1 - \alpha) \big( P_{\widehat{H}} ((\zeta \circ \lambda)(\theta^*)) - P_{\widehat{H}}( (\zeta \circ \lambda)(\theta_t)) \big) + K_m \alpha^2 \ell^2 D^2_{\zeta} \Big), \label{ineq:OIR_prox_conv_rate2}
\end{align}
\normalsize
%
where inequality \eqref{ineq:conv1_8} results from multiplying both sides of \eqref{ineq:OIR_prox_conv_rate1} by $-1$ and adding $P_{\widehat{H}} ((\zeta \circ \lambda)(\theta^*)$, and \eqref{ineq:OIR_prox_conv_rate1_2} and \eqref{ineq:OIR_prox_conv_rate1_3} together yield \eqref{ineq:conv1_9}.

\newstuff{
Using \eqref{ineq:OIR_prox_conv_rate1}, \eqref{ineq:OIR_prox_conv_rate1_2}, \eqref{ineq:OIR_prox_conv_rate1_3}, and \eqref{ineq:OIR_prox_conv_rate2}, we now analyze the sequence $\{ \alpha_t \}_{t \in \mathbb{N}}$ defined in \eqref{eqn:alpha_redef}.
%
%
We first use \eqref{ineq:OIR_prox_conv_rate2} to derive a useful recursive inequality for $\{ \alpha_t \}_{t \in \mathbb{N}}$. Notice that $\alpha_t \geq 0$, for all $t \geq 0$. Now, assume that $\alpha_0 \geq 1$. This implies that $P_{\widehat{H}}((\zeta \circ \lambda)(\theta^*)) - P_{\widehat{H}}((\zeta \circ \lambda)(\theta_0)) \geq 2 K_m \ell^2 D^2_{\zeta}$, so the minimum in \eqref{ineq:OIR_prox_conv_rate2} is attained when $\alpha = 1$. But then $\alpha_1 \leq 1 / 2$. Since this argument is independent of the choice of $t$, we can assume without loss of generality that $\alpha_t \leq 1$, for all $t \geq 0$, by simply discarding $\alpha_0$ if it is greater than 1.
}

We next show $\alpha_{t+1} \leq \alpha_t$, for all $t \geq 0$. Since $\alpha_t \leq 1$, $\alpha_t$ is always the minimizer of the right-hand side of \eqref{ineq:OIR_prox_conv_rate2}, which can be seen by setting the derivative with respect to $\alpha$ equal to 0 and solving for $\alpha$. Substituting $\alpha_t$ into \eqref{ineq:OIR_prox_conv_rate2}, we see that
\small
\begin{align}
    P_{\widehat{H}}((\zeta &\circ \lambda)(\theta^*)) - P_{\widehat{H}}((\zeta \circ \lambda)(\theta_{t+1})) \nonumber \\
    &\leq \Big( 1 - \frac{ P_{\widehat{H}}((\zeta \circ \lambda)(\theta^*)) - P_{\widehat{H}}((\zeta \circ \lambda)(\theta_t)) }{2 K_m \ell^2 D^2_{\zeta}} \Big) \big( P_{\widehat{H}} ((\zeta \circ \lambda)(\theta^*)) - P_{\widehat{H}}( (\zeta \circ \lambda)(\theta_t)) \big) \nonumber \\
    & \hspace{1cm} + \Big( \frac{ P_{\widehat{H}}((\zeta \circ \lambda)(\theta^*)) - P_{\widehat{H}}((\zeta \circ \lambda)(\theta_t)) }{2 K_m \ell^2 D^2_{\zeta}} \Big)^2 K_m \ell^2 D^2_{\zeta} \nonumber \\
    &\labelrel={eqn:conv1_10} \Big( 1 - \frac{ P_{\widehat{H}}((\zeta \circ \lambda)(\theta^*)) - P_{\widehat{H}}((\zeta \circ \lambda)(\theta_t)) }{2 K_m \ell^2 D^2_{\zeta}} \Big) \big( P_{\widehat{H}} ((\zeta \circ \lambda)(\theta^*)) - P_{\widehat{H}}( (\zeta \circ \lambda)(\theta_t)) \big) \nonumber \\
    & \hspace{1cm} + \frac{ P_{\widehat{H}}((\zeta \circ \lambda)(\theta^*)) - P_{\widehat{H}}((\zeta \circ \lambda)(\theta_t)) }{4 K_m \ell^2 D^2_{\zeta}} \big( P_{\widehat{H}}((\zeta \circ \lambda)(\theta^*)) - P_{\widehat{H}}((\zeta \circ \lambda)(\theta_t)) \big) \nonumber \\
    &\labelrel={eqn:conv1_11} \Big( 1 - \frac{ P_{\widehat{H}}((\zeta \circ \lambda)(\theta^*)) - P_{\widehat{H}}((\zeta \circ \lambda)(\theta_t)) }{4 K_m \ell^2 D^2_{\zeta}} \Big) \big( P_{\widehat{H}} ((\zeta \circ \lambda)(\theta^*)) - P_{\widehat{H}}( (\zeta \circ \lambda)(\theta_t)) \big), \label{ineq:OIR_prox_conv_rate3}
\end{align}
\normalsize
where \eqref{eqn:conv1_10} results by noticing that one of the $K_m \ell^2 D_{\zeta}^2$ terms cancels and \eqref{eqn:conv1_11} can be obtained by factoring out the $P_{\widehat{H}} ((\zeta \circ \lambda)(\theta^*)) - P_{\widehat{H}}( (\zeta \circ \lambda)(\theta_t))$ term and simplifying. Dividing both sides of \eqref{ineq:OIR_prox_conv_rate3} by $2 K_m \ell^2 D_{\zeta}^2$ shows that $\alpha_{t+1} \leq \alpha_t$.

Now, dividing both sides of \eqref{ineq:OIR_prox_conv_rate3} by $4 K_m \ell^2 D^2_{\zeta}$ yields the recursive inequality
\begin{align*}
    \frac{\alpha_{t+1}}{2} \leq \Big( 1 - \frac{\alpha_t}{2} \Big) \frac{\alpha_t}{2},
\end{align*}
which implies
\begin{align*}
    \frac{2}{\alpha_{t+1}} \geq \frac{1}{\big( 1 - \frac{\alpha_t}{2} \big) \frac{\alpha_t}{2}} = \frac{2 \big( 1 - \frac{\alpha_t}{2} \big) + \alpha_t }{\big( 1 - \frac{\alpha_t}{2} \big) \alpha_t} = \frac{2}{\alpha_t} + \frac{1}{1 - \frac{\alpha_t}{2}} \geq \frac{2}{\alpha_t} + 1 \geq \frac{2}{\alpha_0} + t.
\end{align*}
Since $\alpha_0 \leq 1$, this gives us that
\begin{align}
    \frac{\alpha_{t+1}}{2} \leq \frac{1}{t + \frac{2}{\alpha_0}} \leq \frac{1}{t + 2}. 
\end{align}
Multiplying both sides by $4 K_m \ell^2 D^2_{\zeta}$ finally yields, for all $t \geq 0$,
\begin{align}
    q(\theta^*) - q(\theta_t) = P_{\widehat{H}}((\zeta \circ \lambda)(\theta^*)) - P_{\widehat{H}}((\zeta \circ \lambda)(\theta_t)) \leq \frac{4 K_m \ell^2 D^2_{\zeta}}{t + 1}. \label{ineq:conv1_13}
\end{align}
This establishes the convergence rate result for \eqref{eqn:OIR_prox_ascent}. We finish the proof by using this result to derive the corresponding rate for the OIR projected gradient descent scheme \eqref{eqn:prox_update_rho}. Since $q(\theta) = 1 / \rho(\theta)$ and $K_m = K / m^2$, \eqref{ineq:conv1_13} implies that
\small
\begin{align*}
    \frac{\rho(\theta_t) - \rho(\theta^*)}{M^2} &\leq  \frac{\rho(\theta_t) - \rho(\theta^*)}{\rho(\theta^*) \rho(\theta_t)} = \frac{1}{\rho(\theta^*)} - \frac{1}{\rho(\theta_t)}
    = q(\theta^*) - q(\theta_t) \leq \frac{4 K_m \ell^2 D^2_{\zeta}}{t + 1} = \frac{4 K \ell^2 D^2_{\zeta}}{m^2 (t + 1)}.
\end{align*}
\normalsize
Since $K = \max \{ m^2 L, M^2 m^2 L \}$ and letting $L_1 = \max \{ L, M^2 L \}$, we have
\begin{equation*}
    \rho(\theta_t) - \rho(\theta^*) \leq \frac{M^2}{m^2} \frac{4 K \ell^2 D^2_{\zeta}}{t + 1} = \frac{4 M^2 L_1 \ell^2 D^2_{\zeta} }{ t + 1 }.
\end{equation*}
\end{proof}

Coupled with Theorem \ref{thm:stationarity_implies_global_optimality_info_ratio}, this result provides a non-asymptotic convergence rate to \textit{global optimality} for algorithms solving the OIR minimization problem \eqref{stat_implies_glob:1}. 
\newstuff{
\begin{remark*}
While the dependence on the hyperparameter $\kappa$ does not appear explicitly in the convergence rate, it does implicitly influence the rate. In particular, as $\kappa \rightarrow \infty$, the objective $\rho(\theta)$ becomes arbitrarily close to the constant function with value 0. Therefore, the suboptimality gap converges more quickly as $\kappa \rightarrow \infty$, since the possible variation of $\rho(\theta)$ about 0 goes to null as $\rho(\theta)$ gets closer to the constant function with value 0. Therefore, if one multiplies the OIR objective by $\kappa$, one obtains the scaled OIR objective $\kappa \rho(\theta) = J(\theta) / \left( 1 + \frac{H(d_{\theta})}{\kappa} \right) \rightarrow J(\theta)$ as $\kappa \rightarrow \infty$. However, altering the objective in this way also changes the behavior of the RHS of the rate given in inequality (5.6). To see this, notice that we can analyze this situation by applying Theorem 5.5 with $J(\theta)$ in the definition of $\rho(\theta)$ replaced by $\kappa J(\theta)$ -- i.e., we simply scale our costs by $\kappa$. If we recall the definition $M = \max_{\theta} \rho(\theta)$ from Theorem 5.5, however, then as $\kappa \rightarrow \infty$ we have that $M = \max_{\theta} \kappa J(\theta) / \left( \kappa + H(d_{\theta}) \right) \rightarrow \max_{\theta} J(\theta)$. This means that, as $\kappa \rightarrow \infty$, its effect on the convergence rate disappears from the RHS of inequality (5.6), leaving us with a standard $\mathcal{O}(1/t)$ rate.
\end{remark*}
}

\begin{remark*} \label{rem:information_bound}
When compared with the corresponding result in \cite{zhang2020variational}, to which it is closely related, the bound \eqref{ineq:OIR_prox_bound} of Theorem \ref{thm:OIR_convergence_rate} contains an interesting dependence on the user-specified $\kappa$, the policy class $\{ \pi_{\theta} \}_{\theta \in \Theta}$, and the underlying MDP.
The presence of $M = \max_{\theta \in \Theta} \rho(\theta) = \max_{\theta} [ J(\theta) / (\kappa + H(d_{\theta}))]$ in the bound \eqref{ineq:OIR_prox_bound} suggests that the convergence rate depends on the value of $\kappa$ as well as the minimal possible value of $H(d_{\theta})$ over $\theta \in \Theta$. To see why, let $C = \max_{\theta \in \Theta} J(\theta)$ and notice that
\begin{equation} 
    M \leq \max_{\theta \in \Theta} \frac{ C }{ \kappa + H(d_{\theta}) } = \frac{ C }{ \kappa + \min_{\theta \in \Theta} H(d_{\theta}) }.
\end{equation}
When the MDP dynamics and policy class are such that $\min_{\theta \in \Theta} H(d_{\theta})$ is large, then $M$ will be closer to 0, yielding a tighter bound in \eqref{ineq:OIR_prox_bound}. This suggests that it may be easier to optimize the OIR over MDPs and/or policy classes that tend to be ``more ergodic''. When both $\kappa$ and $\min_{\theta \in \Theta} H(d_{\theta})$ are close to 0, on the other hand, $M$ may be very large, resulting in a looser bound in \eqref{ineq:OIR_prox_bound}. This highlights the practical usefulness of the constant $\kappa$, as choosing larger $\kappa$ values can be used to \textit{smooth} the objective function $\rho(\theta)$ and thereby lead to stabler convergence when optimizing the OIR over MDPs and policy classes that tend to be ``less ergodic''.
%
%
\end{remark*}

In the preceding theorem, we assume ``exact policy gradient,'' or zero stochastic approximation error. Note this assumption is limited to Theorem \ref{thm:OIR_convergence_rate}, whereas Theorem \ref{thm:actor_conv} below allows stochastic approximation error and Theorem \ref{thm:stationarity_implies_global_optimality_info_ratio} above is independent of estimation issues. Though this assumption is a drawback for Theorem \ref{thm:OIR_convergence_rate}, we highlight that it allows us to succinctly focus on a core insight of this work: hidden quasiconcavity unlocks an information-dependent convergence rate to global optimality. We also note that, for REINFORCE-like algorithms like those considered in Theorem \ref{thm:OIR_convergence_rate}, long rollouts enable more accurate gradient estimates, for which the existing assumptions approximately apply. A precise treatment of gradient estimation error versus rollout length is an important direction future work, and we expect it to involve extending the analysis in \cite{zhang2021beyond} to the OIR problem.


\subsection{Actor-Critic Convergence} \label{sec:as_conv}

We conclude this section by proving almost sure (a.s.) convergence of IDAC to a neighborhood of a stationary point of \eqref{stat_implies_glob:1}. By Theorem \ref{thm:stationarity_implies_global_optimality_info_ratio}, this implies IDAC converges a.s. to a neighborhood of a \textit{global} optimum. This is \textit{much} stronger than existing asymptotic results for actor-critic schemes, which typically guarantee convergence to a neighborhood of a local optimum or saddle point \cite{bhatnagar2009natural, zhang2020global, Agarwal_Kakade_Lee_Mahajan_2020}.
We analyze the algorithm as given in Algorithm \ref{alg:irac} under the assumption that $\tau_t = \alpha_t$, for all $t \geq 0$, that $K = 1$, and with the addition of a projection operation to the policy update:
\begin{equation} \label{eqn:irac_projected_policy_update}
\small
\theta_{t+1} = \Gamma \Big[ \theta_t - \beta_t \frac{ \delta^J_t (\kappa + \mu^H_t) - \mu^J_t \delta^H_t }{ \left( \kappa + \mu^H_t \right)^2 } \nabla \log \pi_{\theta_t}(a_t | s_t) \Big],
\normalsize
\end{equation}
where $\Gamma : \mathbb{R}^d \rightarrow \Theta$ maps any parameter $\theta \in \mathbb{R}^d$ back onto the compact set $\Theta \subset \mathbb{R}^d$ of permissible policy parameters. This projection, which is common in the actor-critic and broader two-timescale stochastic approximation literatures (see, e.g., \cite{kushner03, borkar2008stochastic, bhatnagar2009natural}) is for purposes of theoretical analysis, and is typically not needed in practice.
%
%
In addition to Assumption \ref{assum:actor_conditions}, we impose the following:
\begin{assumption} \label{assum:stepsizes}
Stepsizes \small $\{ \alpha_t \}, \{ \beta_t \}$ \normalsize satisfy \small $\sum_t\alpha_t=\sum_t\beta_t=\infty,~\sum_t\alpha_t^2+\beta_t^2<\infty, \lim_t \frac{\beta_t}{\alpha_t} = 0.$ \normalsize
\end{assumption}
\begin{assumption}\label{assum:critic_conditions}
The value function approximators $v_{\omega}$ are linear, i.e., $v_{\omega}(s)=\omega^{\top}\phi(s)$, where $\phi(s)=[\phi_1(s) \  \cdots \ \phi_K(s)]^\top\in\mathbb{R}^K$ is the feature vector associated with $s \in \mathcal{S}$. The feature vectors $\phi(s)$ are uniformly bounded for any $s \in \mathcal{S}$, and the feature matrix $\Phi=[\phi(s)]_{s \in \mathcal{S}}^\top \in \mathbb{R}^{|\mathcal{S}| \times K}$ has full column rank. For any $u \in \mathbb{R}^K$, $\Phi u \neq \bm{1}$, where $\bm{1}$ is the vector of all ones.
\end{assumption}
Assumptions \ref{assum:actor_conditions}, \ref{assum:stepsizes}, and \ref{assum:critic_conditions} are standard in two-timescale convergence analyses for actor-critic algorithms \cite{bhatnagar2009natural}. 

To prepare for the proof of Theorem \ref{thm:actor_conv}, the main result of this section, we first prove Lemmas \ref{lemma:critics} and \ref{lemma:bias_pg}.
Our analysis leverages the average-reward actor-critic results in \cite{bhatnagar2009natural} as well as the results for ratio optimization actor-critic in \cite{suttle2021reinforcement}. For a given policy parameter $\theta$, let 
$D_\theta = \text{diag} ( d_{\theta} ) \in \mathbb{R}^{| \mathcal{S} | \times | \mathcal{S} |}$ denote the matrix with the elements of $d_{\theta}$ along the diagonal and zeros everywhere else. Define the state cost vector for the average-cost MDP $(\mathcal{S}, \mathcal{A}, p, c)$ to be $c_\theta = [c_{\theta}(s)]_{ s \in \mathcal{S} }^\top \in \mathbb{R}^{| \mathcal{S} |}$, where $c_\theta(s) = \sum_{a \in \mathcal{A}}\pi_\theta(a | s) c(s, a)$. Similarly, let $r_\theta = [- \log d_{\theta}(s)]_{ s \in \mathcal{S} }^\top \in \mathbb{R}^{| \mathcal{S} |}$ denote the state reward vector for the shadow MDP $(\mathcal{S}, \mathcal{A}, p, r)$, where $r(s, a) = - \log d_{\theta}(s)$. Note that the ergodicity condition of Assumption \ref{assum:actor_conditions} implies that $d_{\theta}(s) > 0$, for all $s \in \mathcal{S}, \theta \in \Theta$, so $r(s, a)$ is always defined and finite. Finally, let $P_\theta \in \mathbb{R}^{| \mathcal{S} | \times | \mathcal{S} |}$  denote the state transition probability matrix under policy $\pi_\theta$, i.e., $P_{\theta} ( s' |  s) = \sum_{a \in \mathcal{A} } \pi_{\theta} ( a | s) p(s' | s, a)$, for any $s, s' \in \mathcal{S}$.  
We first show convergence of the critics.




\begin{lemma} \label{lemma:critics}
    Under Assumption \ref{assum:critic_conditions}, given a fixed policy parameter $\theta \in \Theta$, the critic updates in lines 4, 6, 13, 14 of Algorithm \ref{alg:irac} converge as follows: $\lim_{t \rightarrow \infty} \mu^J_t = J(\theta)$ a.s., $\lim_{t \rightarrow \infty} \mu^H_t = H(d_{\theta})$ a.s., $\lim_{t \rightarrow \infty} \omega^J_t = \omega^J_{\theta}$ a.s., and $\lim_{t \rightarrow \infty} \omega^H_t = \omega^H_{\theta}$ a.s., where $\omega^J_{\theta}$ and $\omega^H_{\theta}$ are, respectively, the unique solutions to
    \small
    \begin{align}
    \Phi^{ \top } D_{\theta} \big[ c_{\theta} - J(\theta) \cdot \bm{1} + P_\theta( \Phi \omega^J ) - \Phi \omega^J \big] & = \bm{0}, \nonumber \\
    \Phi^{ \top } D_{\theta} \big[ r_{\theta} - H(d_{\theta}) \cdot \bm{1} + P_\theta( \Phi \omega^H ) - \Phi \omega^H \big] & = \bm{0}. \nonumber
    \end{align}
    \normalsize
\end{lemma}

\begin{proof}
Since the policy $\pi_{\theta}$ is held fixed and the shadow MDP reward $- \log d_{\theta}(s)$ can be exactly evaluated, for any $s \in \mathcal{S}$, the proof of Lem.~4 in \cite{bhatnagar2009natural} can be applied separately to the average-cost recursions in lines 4, 9, and 13 and the shadow MDP recursions in lines 6, 10, and 14 of Algorithm \ref{alg:irac} to obtain the result.
\end{proof}

As in Lem.~5 of \cite{bhatnagar2009natural}, this result shows that the sequences $\{ \omega^J_t \}$ and $\{ \omega^H_t \}$ converge a.s. to the limit points $\omega^J_{\theta}$ and $\omega^H_{\theta}$ of the TD(0) algorithm with linear function approximation for their respective MDPs.
Due to the use of linear function approximation, when used in the policy update step the value function estimates $v^J_{\theta} = \Phi \omega^J_{\theta}$ and $v^H_{\theta} = \Phi \omega^H_{\theta}$ may result in biased gradient estimates. Similar to the bias characterization given in Lem.~4 in \cite{bhatnagar2009natural}, this bias can be characterized as follows.


\begin{lemma}\label{lemma:bias_pg}
	Fix $\theta\in\Theta$. Let $\delta^{\theta, J}_t = c(s_t, a_t) - J(\theta) + \phi(s_{t+1})^\top \omega^J_{\theta} - \phi(s_t)^\top \omega^J_{\theta}$
    and
	$\delta^{\theta, H}_t = - \log d_{\theta}(s_t) - H(d_{\theta}) + \phi(s_{t+1})^\top \omega^H_{\theta} - \phi(s_t)^\top \omega^H_{\theta}$
	denote the stationary estimates of the TD-errors at time $t$. Let
	$\overline{v}^J_{\theta} = \mathbb{E}_{\pi_{\theta}} \Big[ c(s,a) - J(\theta) + \phi(s')^\top \omega^J_{\theta} \Big]$
    and
	$\overline{v}^H_{\theta} = \mathbb{E}_{\pi_{\theta}} \Big[ - \log d_{\theta}(s) - H(d_{\theta}) + \phi(s')^\top \omega^H_{\theta} \Big]$.
	Finally, let
    $\epsilon^J_{\theta} = \sum_{s \in \mathcal{S}} d_{\theta}(s) \left[ \nabla_{\theta} \overline{v}^J_{\theta}(s) - \nabla_{\theta} \phi(s)^{\top} \omega^J_{\theta} \right]$
    and
    $\epsilon^H_{\theta} = \sum_{s \in \mathcal{S}} d_{\theta}(s) \left[ \nabla_{\theta} \overline{v}^H_{\theta}(s) - \nabla_{\theta} \phi(s)^{\top} \omega^H_{\theta} \right]$.
	We then have that
	\small
	\begin{align*}
	\mathbb{E}_{\pi_{\theta}} \left[ \frac{ \delta^{\theta, J}_t \left[ \kappa + H(d_{\theta}) \right] - J(\theta) \delta^{\theta, H}_t}{ \left[ \kappa + H(d_{\theta}) \right]^2} \nabla \log \pi_{\theta}(a_t | s_t) \right] = \nabla \rho(\theta) + \frac{ \epsilon^J_{\theta} \left[ \kappa + H(d_{\theta}) \right] - J(\theta) \epsilon^H_{\theta}}{ \left[ \kappa + H(d_{\theta}) \right]^2}.
	\end{align*}
	\normalsize
\end{lemma}

\begin{proof}
By \cite[Lemma~4]{bhatnagar2009natural} and Theorem \ref{thm:entropy_gradient},
%
%
$\mathbb{E}_{\pi_{\theta}} \left[ \delta^{\theta, J}_t \nabla \log \pi_{\theta}(a_t | s_t) \right] = \nabla J(\theta) + \epsilon^J_{\theta}$ and $\mathbb{E}_{\pi_{\theta}} \left[ \delta^{\theta, H}_t \nabla \log \pi_{\theta}(a_t | s_t) \right] = \nabla H(d_{\theta}) + \epsilon^H_{\theta}$.
This implies that
\small
\begin{align*}
    \mathbb{E}_{\pi_{\theta}} & \left[ \frac{ \delta^{\theta, J}_t \left[ \kappa + H(d_{\theta}) \right] - J(\theta) \delta^{\theta, H}_t}{ \left[ \kappa + H(d_{\theta}) \right]^2} \nabla \log \pi_{\theta}(a_t | s_t) \right] \\
    &= \frac{ \left[ \kappa + H(d_{\theta}) \right] \mathbb{E}_{\pi_{\theta}} \left[ \delta^{\theta, J}_t \nabla \log \pi_{\theta}(a_t | s_t) \right] - J(\theta) \mathbb{E}_{\pi_{\theta}} \left[ \delta^{\theta, H}_t \nabla \log \pi_{\theta}(a_t | s_t) \right]}{ \left[ \kappa + H(d_{\theta}) \right]^2 } \\
    &= \frac{ \left[ \kappa + H(d_{\theta}) \right] \left( \nabla J(\theta) + \epsilon^J_{\theta} \right) - J(\theta) \left( \nabla H(d_{\theta}) + \epsilon^H_{\theta} \right) }{ \left[ \kappa + H(d_{\theta}) \right]^2 } \\
    &= \nabla \rho(\theta) + \frac{ \epsilon^J_{\theta} \left[ \kappa + H(d_{\theta}) \right] - J(\theta) \epsilon^H_{\theta}}{ \left[ \kappa + H(d_{\theta}) \right]^2},
\end{align*}
\normalsize
which completes the proof.
\end{proof}

We now establish convergence of the actor step, and thus the actor-critic algorithm. Given any continuous function $f : \Theta \to \mathbb{R}^d$, define the function $\hat{\Gamma}(\cdot)$ using the projection operator $\Gamma$ to be $\hat{\Gamma}(f(\theta))=\lim_{\eta \rightarrow 0^+} \left [ \Gamma(\theta+\eta \cdot f(\theta))-\theta\right ] \big/ \eta$.
%
%
%
Define
\begin{equation} \label{eqn:error_def}
    \epsilon_{\theta} = \frac{ \epsilon^J_{\theta} \left[ \kappa + H(d_{\theta}) \right] - J(\theta) \epsilon^H_{\theta} }{ \left[ \kappa + H(d_{\theta}) \right]^2}.
\end{equation}
Consider the ODEs

\begin{minipage}{0.43\textwidth}
\begin{equation} \label{eqn:actor_ode}
\dot{\theta} = \hat{\Gamma} (\nabla \rho(\theta)), 
\end{equation}
\end{minipage}
\hfill
\begin{minipage}{0.43\textwidth}
\begin{equation} \label{eqn:biased_actor_ode}
\dot{\theta} = \hat{\Gamma} (\nabla \rho(\theta) + \epsilon_{\theta}). 
\end{equation}
\end{minipage}
\vspace{2mm}

\noindent Notice that, by the definition of $\hat{\Gamma}$, the right-hand side of \eqref{eqn:actor_ode} is simply $\Gamma(\nabla \rho(\theta))$ when there exists $\eta_0 > 0$ such that $\theta + \eta \nabla \rho(\theta) \in \Theta$, for all $\eta < \eta_0$. When such an $\eta_0$ does not exist, $\hat{\Gamma}(\nabla \rho(\theta))$ can be interpreted as the projected ODE $\dot{\theta} = \nabla \rho(\theta) + z(\theta)$, where $z(\theta)$ is the minimal force necessary to project $\theta$ back onto $\Theta$. Similar statements hold for \eqref{eqn:biased_actor_ode}. For further discussion of the definition of $\hat{\Gamma}$ and related results, see p.~191 of \cite{kushner1978stochastic}. For the projected ODE interpretation, see \S4.3 of \cite{kushner03}.

We now present the main result of this subsection, which establishes convergence of the actor-critic algorithm. Its proof follows that of Theorem~1 in \cite{bhatnagar2009natural}, with key modifications to accommodate complications arising from the fact that the objective to be minimized is a ratio; specifically, we ensure that: (i) the resulting noise terms are indeed asymptotically negligible, and (ii) the Lipschitz properties of the gradient $\nabla \rho(\theta)$ necessary for the ODE analysis are satisfied.
\begin{theorem}\label{thm:actor_conv}
	Let $\mathcal{Z}$ denote the set of asymptotically stable equilibria of the ODE \eqref{eqn:actor_ode}. Given any $\varepsilon > 0$, define $\mathcal{Z}^{\varepsilon} = \{ z \ | \ \inf_{z' \in \mathcal{Z}} \| z - z' \| \leq \varepsilon \}$. For any $\theta \in \Theta$, let $\varepsilon_{\theta}$ be defined as in \eqref{eqn:error_def}. Under Assumptions \ref{assum:actor_conditions}, \ref{assum:stepsizes}, and \ref{assum:critic_conditions}, given any $\varepsilon > 0$, there exists $\delta > 0$ such that, for $\{ \theta_t \}$  obtained from Algorithm \ref{alg:irac} with projection \eqref{eqn:irac_projected_policy_update}, if $\sup_t \| \epsilon_{\theta_t} \| < \delta$, then $\theta_t \to \mathcal{Z}^\varepsilon$ a.s. as $t \to \infty$.
\end{theorem}
\begin{proof}
%
%
\newstuff{
Let $\mathcal{F}_t = \sigma( \theta_k, k \leq t )$ denote the $\sigma$-algebra generated by the $\theta$-iterates up to time $t$.
Define $\delta_t = \frac {\delta^J_t \left[ \kappa + \mu^H_t \right] - \mu^J_t \delta^H_t}{\left[ \kappa + \mu^H_t \right]^2}$ and $\delta^{\theta}_t = \frac{ \delta^{J, \theta}_t \left[ \kappa + H(d_{\theta}) \right] - J(\theta) \delta^{H, \theta}_t }{ \left[ \kappa + H(d_{\theta}) \right]^2}$.
%
%
In addition, define the noise terms 
$M^{(1)}_t = \delta_t \nabla \log \pi_{\theta_t} (a_t | s_t) - \mathbb{E} \left[ \delta_t \nabla \log \pi_{\theta_t} (a_t | s_t) \ | \ \mathcal{F}_t \right]$ and $M^{(2)}_t = \mathbb{E} [ ( \delta_t - \delta^{\theta_t}_t ) \nabla \log \pi_{\theta_t} (a_t | s_t) \ | \ \mathcal{F}_t ]$.
%
%
Finally, define the function $h(\theta_t) = \mathbb{E}_{\pi_{\theta_t}} [ \delta^{\theta_t}_t \nabla \log \pi_{\theta_t} (a_t | s_t) \ | \ \mathcal{F}_t ] = \mathbb{E}_{\pi_{\theta_t}} [ \delta^{\theta_t}_t \nabla \log \pi_{\theta_t} (a_t | s_t) ]$,
%
%
which is the gradient expression from Lemma \ref{lemma:bias_pg}. Note that simultaneously taking an expectation with respect to $\pi_{\theta_t}$ and conditioning on $\mathcal{F}_t$ is redundant, so we can suppress one or the other in our notation without altering the meaning.
%

We can now rewrite the projected actor update \eqref{eqn:irac_projected_policy_update} as
\begin{align*}
    \theta_{t+1} = \Gamma \Big( \theta_t - \beta_t \delta_t \nabla \log \pi_{\theta_t} (a_t | s_t) \Big) = \Gamma \left( \theta_t - \beta_t \left[ h(\theta_t) + M^{(1)}_t + M^{(2)}_t \right] \right).
\end{align*}
We show that this update scheme asymptotically tracks the ODE \eqref{eqn:biased_actor_ode} a.s. by demonstrating that the noise terms $\{ M^{(1)}_t \}$ form an a.s. bounded martingale difference sequence, that the terms $\{ M^{(2)}_t \}$ are asymptotically negligible, and that $h$ is Lipschitz and thus the ODE is well-posed.

Since $\delta_t \rightarrow \delta^{\theta_t}_t$ a.s. by Lemma \ref{lemma:critics}, we have that $M^{(2)}_t \rightarrow 0$ a.s., so the noise terms $\{ M^{(2)}_t \}$ are indeed asymptotically negligible. Next, recall the tower property of conditional expectations: for any $\mathcal{F}$-measurable random variable $X$ and any sub-$\sigma$-algebras $\mathcal{G} \subset \mathcal{H} \subset \mathcal{F}$, we have $\mathbb{E} \left[ \mathbb{E} \left[ X | \mathcal{G} \right] | \mathcal{H} \right] = \mathbb{E} \left[ \mathbb{E} \left[ X | \mathcal{H} \right] | \mathcal{G} \right] = \mathbb{E} \left[ X | \mathcal{G} \right]$. Since $\mathcal{F}_t \subset \mathcal{F}_{t+1}$, for all $t \geq 0$, this implies that, for all $t \geq 0$,
\small
\begin{align*}
    \mathbb{E} \left[ M^{(1)}_{t+1} | \mathcal{F}_t \right] &= \mathbb{E} \left[ \delta_{t+1} \nabla \log \pi_{\theta_{t+1}} (a_{t+1} | s_{t+1}) - \mathbb{E} \left[ \delta_{t+1} \nabla \log \pi_{\theta_{t+1}} (a_{t+1} | s_{t+1}) \ | \ \mathcal{F}_{t+1} \right] \ | \ \mathcal{F}_t \right] \\
    &= \mathbb{E} \left[ \delta_{t+1} \nabla \log \pi_{\theta_{t+1}} (a_{t+1} | s_{t+1}) \ | \ \mathcal{F}_t \right] \\
    & \hspace{2cm} - \mathbb{E} \left[ \mathbb{E} \left[ \delta_{t+1} \nabla \log \pi_{\theta_{t+1}} (a_{t+1} | s_{t+1}) \ | \ \mathcal{F}_{t+1} \right] \ | \ \mathcal{F}_t  \right] \\
    &= \mathbb{E} \left[ \delta_{t+1} \nabla \log \pi_{\theta_{t+1}} (a_{t+1} | s_{t+1}) \ | \ \mathcal{F}_t \right] - \mathbb{E} \left[ \delta_{t+1} \nabla \log \pi_{\theta_{t+1}} (a_{t+1} | s_{t+1}) \ | \ \mathcal{F}_t \right] \\
    &= 0,
\end{align*}
\normalsize
so $\{ M^{(1)}_t \}$ is an $\mathcal{F}$-martingale difference sequence, where $\mathcal{F}$ is the filtration $\mathcal{F} = \{ \mathcal{F}_t \}$.

To see that $M^{(1)}_t$ is a.s. bounded, first notice that $\mu^H_0 > 0$, and $0 < \inf_t d_{\theta_t}(s_t) \leq \sup_t d_{\theta_t}(s_t) \leq 1$, so $\{ \mu^H_t \}$ is uniformly bounded both above and below away from zero. A similar argument applies to $\{ \mu^J_t \}$. Coupled with a.s. boundedness of $\{ \omega^J_t \}$ and $\{ \omega^H_t \}$, this implies that $\{ \delta_t \}$ and thus $\{ M^{(1)}_t \}$ are a.s. bounded. As discussed in \S\S2.1-2.2 of \cite{borkar2008stochastic}, the facts that $\{ M^{(1)}_t \}$ is a.s. bounded martingale difference noise and $\{ M^{(2)}_t \}$ is asymptotically negligible ensure that, so long as the right-hand side of the ODE \eqref{eqn:biased_actor_ode} is Lipschitz, the iterates generated by \eqref{eqn:irac_projected_policy_update} will asymptotically track~it.

To see that $h$ is Lipschitz in $\theta$, first rewrite
\small
\begin{align*}
    h(\theta_t) = \frac{1}{\left[ \kappa + H(d_{\theta_t}) \right]^2} \Big( & \left[ \kappa + H(d_{\theta_t}) \right] \mathbb{E}_{\pi_{\theta_t}} \left[ \delta^{J, \theta_t}_t \nabla \log \pi_{\theta_t} (a_t | s_t) \right] \\
    &- J(\theta_t) \mathbb{E}_{\pi_{\theta_t}} \left[ \delta^{H, \theta_t}_t \nabla \log \pi_{\theta_t} (a_t | s_t) \right] \Big).
\end{align*}
\normalsize
We verify that each of the component terms in this expression is Lipschitz and bounded. Recall that a function is Lipschitz if it is continuously differentiable with bounded derivatives. 
First, as discussed in the proof of \cite{bhatnagar2009natural}, Lem.~5, $J(\theta)$, $d_{\theta}(s)$, $\nabla \pi_{\theta}(a | s)$, and $\Phi \omega^J_{\theta}$, are all Lipschitz and bounded, for all $s \in \mathcal{S}, a \in \mathcal{A}$. Thus $J(\theta)$ and $\mathbb{E}_{\pi_{\theta}} \left[ \delta^{J, \theta}_t \nabla \log \pi_{\theta} (a_t | s_t) \right]$ are Lipschitz and bounded on $\Theta$. The remaining terms we need to inspect are $\kappa + H(d_{\theta})$, $1 / \left[ \kappa + H(d_{\theta}) \right]^2$, and $\mathbb{E}_{\pi_{\theta}} \left[ \delta^{H, \theta}_t \nabla \log \pi_{\theta} (a_t | s_t) \right]$.

Theorem \ref{thm:entropy_gradient} implies $\nabla H(d_{\theta})$ is continuous and bounded. To see this, notice that
\small
\begin{align*}
    \nabla H(d_{\theta}) &= \mathbb{E}_{\pi_{\theta}} \Big[ \left( - \log d_{\theta}(s) - H(d_{\theta}) \right) \nabla \log \pi_{\theta}(a | s) \Big] \\
    &= \sum_s d_{\theta}(s) \sum_a \pi_{\theta}(a | s) \Big[ \left( - \log d_{\theta}(s) - H(d_{\theta}) \right) \nabla \log \pi_{\theta}(a | s) \Big] \\
    &= \sum_s d_{\theta}(s) \sum_a \nabla \pi_{\theta}(a | s) \Big[ \left( - \log d_{\theta}(s) - H(d_{\theta}) \right) \Big].
\end{align*}
\normalsize

By the ergodicity condition of Assumption \ref{assum:actor_conditions}, we have that $d_{\theta}(s) > 0$, for all $s \in \mathcal{S}$, which means that the $- \log d_{\theta}(s)$ is always defined. Since $d_{\theta}$ is continuous, we furthermore have that $- \log d_{\theta}(s)$ and $H(d_{\theta})$ are both continuous. The gradient $\nabla \pi_{\theta}(a | s)$ is continuous by Assumption \ref{assum:actor_conditions}. Finally, since $\Theta$ is a compact set, we know that $d_{\theta}(s), \nabla \pi_{\theta}(a | s), - \log d_{\theta}(s),$ and $H(d_{\theta})$ remain bounded, implying that $\nabla H(d_{\theta})$ is continuous and bounded, since it is formed by taking products and sums of continuous, bounded functions. $H(d_{\theta})$ is thus Lipschitz and bounded, as is the term $\kappa + H(d_{\theta})$, for any constant $\kappa \geq 0$. Furthermore, since $d_{\theta}(s) > 0$, for all $s \in \mathcal{S}$, and since $\Theta$ is compact, there exists some constant $B$ such that $\inf_{\theta \in \Theta} H(d_{\theta}) = B > 0$. This means that $1 / \left[ \kappa + H(d_{\theta}) \right]^2 \leq 1 / \left[ \kappa + B \right]^2$, for all $\theta \in \Theta$. The term $1 / \left[ \kappa + H(d_{\theta}) \right]^2$ is therefore Lipschitz and bounded, as well.

Finally, notice that $\mathbb{E}_{\pi_{\theta}} \left[ \delta^{H, \theta}_t \nabla \log \pi_{\theta} (a_t | s_t) \right] =$
\small
\begin{align}
    &\mathbb{E}_{\pi_{\theta}} \left[ \left( - \log d_{\theta}(s_t) - H(d_{\theta}) + \phi(s_{t+1})^{\top} \omega^H_{\theta} - \phi(s_t)^{\top} \omega^H_{\theta} \right) \nabla \log \pi_{\theta}(a_t | s_t) \right] \nonumber \\
    &= \sum_s d_{\theta}(s) \sum_a \nabla \pi_{\theta}(a | s) \left[ - \log d_{\theta}(s) - H(d_{\theta}) - \phi(s)^{\top} \omega^H_{\theta}  + \sum_{s'} p(s' | s, a) \phi(s')^{\top} \omega^H_{\theta} \right]. \label{eqn:final_actor_term}
    \normalsize
\end{align}
\normalsize
As discussed above, $d_{\theta}(s), \pi_{\theta}(a | s), - \log d_{\theta}(s),$ and $H(d_{\theta})$ are all continuously differentiable with bounded derivatives on $\Theta$. Furthermore, given that $H(d_{\theta})$ is Lipschitz and bounded both above and away from zero, $\Phi \omega^H_{\theta}$ is Lipschitz and bounded for reasons analogous to those for $\Phi \omega^J_{\theta}$. Expression \eqref{eqn:final_actor_term} is thus Lipschitz and bounded. 

By the foregoing, $h$ is Lipschitz, since it is formed by taking products and sums of Lipschitz, bounded functions.
The ODE \eqref{eqn:biased_actor_ode} is therefore well-posed, and its equilibrium set $\mathcal{Z}$ is well-defined. A similar argument to the one just presented can be used to show that \eqref{eqn:actor_ode} with equilibrium set $\mathcal{Y}$ is also well-posed. The remainder of the arguments in the proof of Lem.~5 in \cite{bhatnagar2009natural} now apply to prove that $\theta_t \rightarrow \mathcal{Y}$ a.s. as $t \rightarrow \infty$, and that, as $\sup_{\theta} \| \epsilon_{\theta} \| \rightarrow 0$, the trajectories of \eqref{eqn:biased_actor_ode} converge to those of \eqref{eqn:actor_ode}. In particular, this implies that, for a given $\varepsilon > 0$, there exists a $\delta > 0$ such that, if $\sup_{\theta} \| \epsilon_{\theta} \| < \delta$, then $\theta_t \rightarrow \mathcal{Z}^{\varepsilon}$ a.s. as $t \rightarrow \infty$.
}
\end{proof}

Combined with Theorem \ref{thm:stationarity_implies_global_optimality_info_ratio}, Theorem \ref{thm:actor_conv} establishes almost sure convergence of IDAC to a neighborhood of a \textit{global} optimum of the OIR minimization problem \eqref{stat_implies_glob:1}.
Note that if the linear approximation and features are expressive enough, then
$\varepsilon$ will be small or even zero.

%% file: sections/Experiments.tex
\section{Experiments} \label{sec:experiments}


\newstuff{
The experimental results presented in this section demonstrate that, when the reward signal is sparse, OIR methods can lead to improved performance when compared with vanilla RL methods. These results provide empirical support to the study of the OIR as an important and useful RL objective, and are meant to be viewed as auxiliary to the purely theoretical results presented above.
To demonstrate the advantages of OIR policy gradient methods over vanilla methods in such settings, we conducted two different sets of experiments on gridworld environments of varying complexity. In the first set of experiments, discussed in \S\ref{subsec:tabular_experiments}, we compared tabular implementations of IDAC and vanilla AC on three relatively small gridworlds. For the second set of experiments, which are discussed in \S\ref{subsec:nn_experiments}, we compared a neural network version of IDAC with the A2C, DQN, and PPO algorithms on a larger, more complex gridworld. All the environments that we considered emit sparse reward signals in the sense that the majority of costs convey no information about the central task of finding the goal state. Specific details regarding the environments and algorithm implementations are provided in \S\ref{subsec:environments} and \S\ref{subsec:implementation}, respectively. On all four gridworlds, OIR policy gradient methods outperform the vanilla RL methods that we tested. We interpret this as illustrating that algorithms minimizing the OIR fall back to maximizing the useful alternative objective $H(d_{\theta})$, representing coverage of the state space, when reward signals are sparse. This contrasts with vanilla methods, which focus on a single objective, $J(\theta)$, and can consequently converge in the sparse-reward setting to suboptimal policies before the state space has been sufficiently explored.


\subsection{Environments} \label{subsec:environments}

Each gridworld is composed of an $n \times m$ grid of states, $\mathcal{S} = \{0, \ldots, n - 1\} \times \{0, \ldots, m - 1\}$, along with a designated start state $s_{\text{start}}$, designated goal state $s_{\text{goal}}$, and a set $B \subset \mathcal{S}$ of blocked states which the agent is not permitted to enter. Episodes are of fixed length $K$, and the agent begins each episode in state $s_{\text{start}}$. In a given state $s = (i, j)$, the agent chooses an action $a \in \{ \text{stay}, \text{ up}, \text{ down}, \text{ left}, \text{ right} \}$. The agent then attempts to move in the direction corresponding to the action selected: if the selected action would move the agent off the grid or into a blocked state, the agent remains in $s$; otherwise, the agent moves into (or remains in) the state corresponding to the action selected. For example, if $a = \text{up}$ is chosen, the agent attempts to move to state $s' = (i, j - 1)$. If $s'$ is off the grid (i.e. $j - 1 < 0$) or $s' \in B$, the agent remains in $s$. Otherwise, the agent transitions to $s'$. Finally, let $\mathcal{A}(s)$ denote the set of all actions at $s$ that do not lead off the grid or into a blocked state; the cost function is then given by:
\begin{equation*}
    c(s, a) =
    \begin{cases*}
        c_{\text{goal}} & if $s = s_{\text{goal}}$ and $a \in \mathcal{A}(s)$, \\
        c_{\text{allowed}} & if $s \neq s_{\text{goal}}$ and $a \in \mathcal{A}(s)$, \\
        c_{\text{blocked}} & if $a \notin \mathcal{A}(s)$,
    \end{cases*}
\end{equation*}
where $0 < c_{\text{goal}} < c_{\text{allowed}} < c_{\text{blocked}}$. A policy minimizing $J(\theta)$ will move as quickly as possible to $s_{\text{goal}}$ while always choosing actions within $\mathcal{A}(s)$. Because of this, when a problem is small enough that the agent can reach the goal state quickly and remain in it for most of the episode, the optimal average cost should be close to 1. A policy minimizing $\rho(\theta)$, on the other hand, will seek to balance minimizing $J(\theta)$ with maximizing $H(d_{\theta})$, while avoiding actions $a \notin \mathcal{A}(s)$.

\subsection{Implementation} \label{subsec:implementation}

For the first set of experiments, we implemented a tabular version of Algorithm \ref{alg:irac}. In order to have a baseline to compare against, we also implemented classic average-cost actor-critic, vanilla AC.
For both algorithms, we used tabular softmax policies:
$$\pi_{\theta}(a_i | s) = \frac{\exp(\theta^T \psi(s, a_i))}{\sum_j \exp(\theta^T \psi(s, a_j))},$$
where $\theta \in \mathbb{R}^{|\mathcal{S}| \cdot |\mathcal{A}|}$ and $\psi : \mathcal{S} \times \mathcal{A} \rightarrow \mathbb{R}^{|\mathcal{S}| \cdot |\mathcal{A}|}$ maps each state-action pair to a unique standard basis vector $e_k \in \mathbb{R}^{|\mathcal{S}| \cdot |\mathcal{A}|}$, where $e_k$ has a $1$ in its $k$th entry and $0$ everywhere else. We similarly used tabular representation for the value functions:
$$v_{\omega}(s) = \omega^T \phi(s),$$
where $\omega \in \mathbb{R}^{|\mathcal{S}|}$ and $\phi : \mathcal{S} \rightarrow \mathbb{R}^{|\mathcal{S}|}$ maps each state $s_i$ to a unique standard basis vector $e_i$.

For the second set of experiments, we implemented IDAC with a categorical policy using two-layer, fully connected neural networks for both the policy and value function approximators, and we compared against the Stable Baselines 3 \cite{raffin2019stable} implementations of A2C, DQN, and PPO with two-layer, fully connected neural networks for all policies and value function approximators.

\subsection{Tabular Experiment Results} \label{subsec:tabular_experiments}

Figures \ref{fig:tabular_gridworlds} shows comparisons of IDAC and vanilla AC on three \texttt{GridWorld} environments with $c_{\text{goal}} = 1, c_{\text{allowed}} = 10,$ and $c_{\text{blocked}} = 100$. To generate these results, 15 instances of each algorithm were run on the corresponding environment, the average cost and entropy were computed for each episode, and the sample means and 95\% confidence intervals for the cost, entropy, and corresponding OIR over the 15 runs were used to generate the learning curves. As the figures show, the OIR algorithm outperforms the vanilla algorithm in every case. In particular, IDAC consistently explores the state space, leading to eventual discovery of the goal state, while vanilla AC quickly becomes deterministic and converges to a suboptimal policy. This supports our interpretation that, in the sparse-reward setting, algorithms minimizing the OIR fall back on maximizing a useful secondary objective, $H(d_{\theta})$; this can provide an advantage over vanilla methods focused solely on minimizing $J(\theta)$.

On all three \texttt{GridWorld} environments, for both the IDAC and vanilla AC algorithms we used actor learning rate $\alpha = 1.8$, critic learning rate $\beta = 2.0$, and geometric mixing rate $\tau = 0.1$. For IDAC, we set $\kappa = 1.0$. We chose these parameters through trial and error. For \texttt{GridWorld1}, we used episode length 200 over 2500 episodes. For \texttt{GridWorld2}, we used episode length 200 over 3000 episodes. Finally, for \texttt{GridWorld3}, we used episode length 300 over 3000 episodes. To facilitate learning, we found it helpful to increase the episode length and number of episodes as the complexity of the problem increased. Figure \ref{fig:tabular_gridworlds} presents the average cost, entropy, and OIR (with $\kappa = 1.0$) for both algorithms as training proceeds. Note that we did not provide optimal benchmarks for these problems using the concave program solver. Since the solver finds the optimal state-action occupancy measure based on the assumption that it is independent of the initial start state, the designated start states inherent in the \texttt{GridWorld} environments causes the solver results to be inaccurate.

\begin{figure} 
    \centering
    \includegraphics[width=\linewidth]{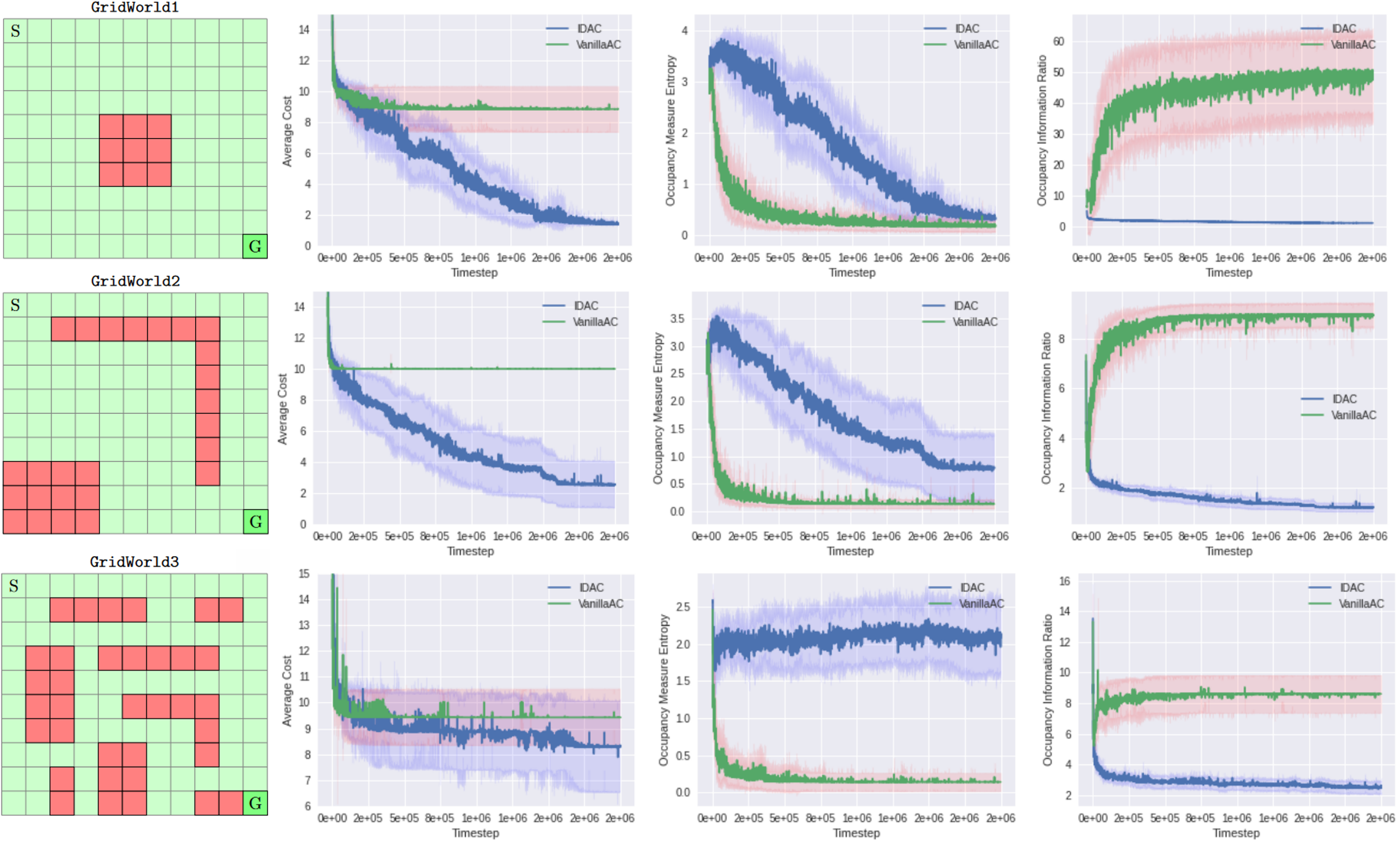}
    \captionof{figure}{Comparison of tabular IDAC and vanilla actor-critic on three different gridworlds. Plots give means and 95\% confidence intervals. Optimal average cost is 1.0.}
    \label{fig:tabular_gridworlds}
\end{figure}

As can be seen from the average cost in all three figures, both algorithms quickly learn to avoid actions moving off the grid or into blocked states, decreasing to an average cost of around 10. On all three environments, vanilla AC gets stuck near 10 for the remainder of training. This corresponds to taking allowed actions, but not attaining the goal state. The IDAC algorithm, on the other hand, clearly spends an increasing amount of time in the goal state, since its cost decreases well below 10. Next, the evolution of the state occupancy measure entropy achieved by the two algorithms during training provides some insight into why vanilla AC fails while IDAC succeeds. Vanilla AC converges fairly quickly to a policy visiting only a small subset of the available states, reflecting overconfidence in its past experience. This is why vanilla AC struggles on these environments, since its policy becomes deterministic before the state space has been sufficiently explored. IDAC, in contrast, maintains policies with relatively high state occupancy measure entropy early on, only decreasing as the algorithm seeks to strike the right balance between cost and entropy. Finally, in all cases IDAC makes clear progress minimizing the OIR, while vanilla AC consistently increases it.

\subsection{Neural Network Experiment Results} \label{subsec:nn_experiments}



\begin{figure} 
    \centering
    \includegraphics[width=\linewidth]{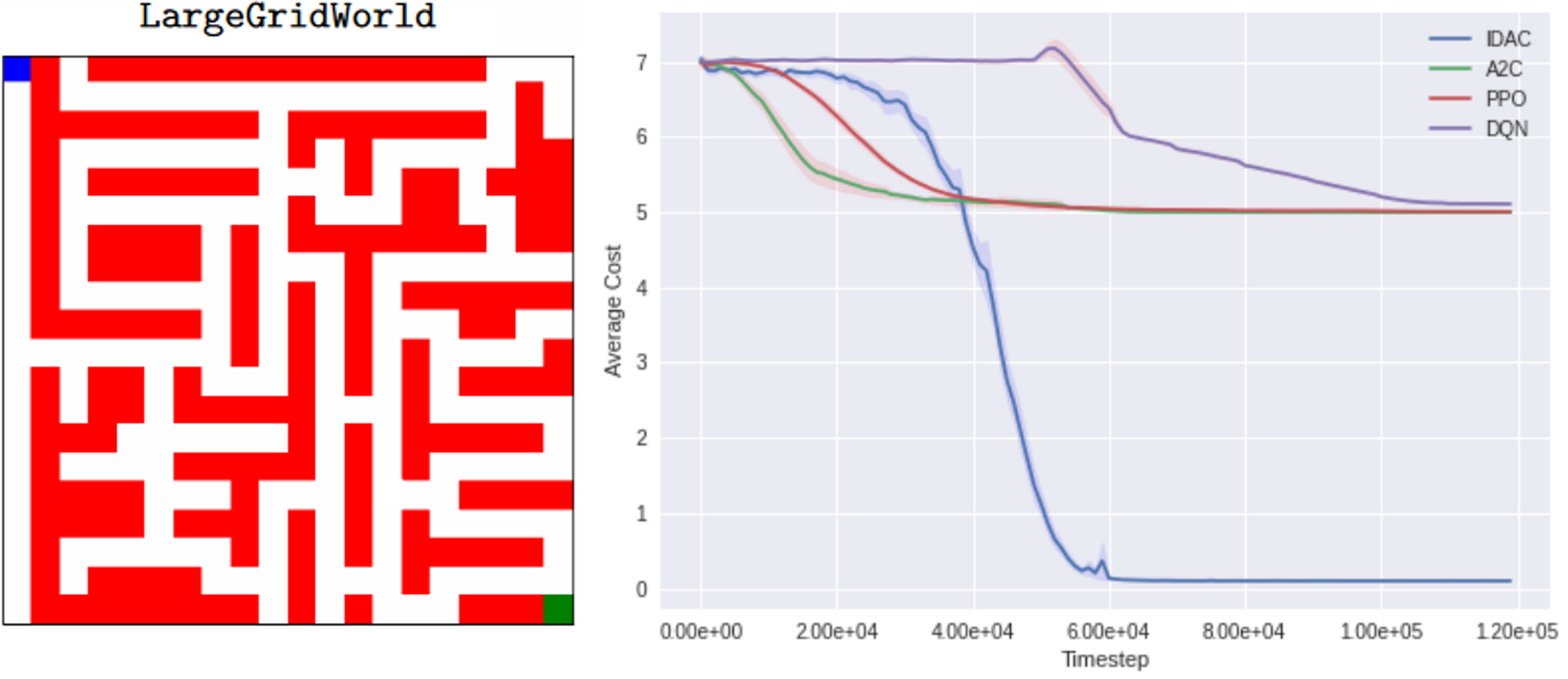}
    \captionof{figure}{Comparison of neural network IDAC with common deep RL methods. Plot gives means and 95\% confidence intervals. Optimal average cost is 0.1. Training took place over $1\mathrm{e}{+6}$ timesteps; no further improvement occurred beyond timestep $1.2\mathrm{e}{+5}$.}
    \label{fig:large_gridworld}
\end{figure}

Figure \ref{fig:large_gridworld} illustrates the performance of neural IDAC and A2C, DQN, and PPO on \texttt{LargeGridWorld} with $c_{\text{goal}} = 0.1,$ $c_{\text{allowed}} = 5,$ and $c_{\text{blocked}} = 10$. To generate the data for these figures, we first trained 48 instances of neural IDAC with different random seeds. We next trained 15 instances of each of the A2C, DQN, and PPO algorithms on the environment. For each algorithm, the average cost was computed for each episode, and the sample means and 95\% confidence intervals were used to create the learning curves. As the figure illustrates, IDAC outperformed all three. Furthermore, none of A2C, DQN, and PPO found the goal state after $1.2\mathrm{e}{+5}$ timesteps.

Hyperparameters $\alpha = 0.0001$, $\beta = 0.0002$, $\tau = 0.1$, and $\kappa = 0.1$ for neural IDAC were selected through trial and error. After finding that increasing the width of the layers improved performance, we used 512 hidden units for each layer in both the policy and value functions. After experimenting with a range of different parameters and detecting no noticeable difference in performance, Stable Baselines' default parameters for A2C, DQN, and PPO were used. This included learning rates $0.0007$ for A2C, $0.0003$ for PPO, and $0.0001$ for DQN, as well as 64-width layers for all networks.

As in the tabular experiments, all algorithms quickly learn to avoid blocked actions. In the case of A2C and PPO, this leads to an average cost of exactly $5$, while for DQN the cost remains slightly above $5$ due to exploration noise lower bounded by $0.05$. Though the optimal cost is $0.1$, once they have converged to these values, they remain there for the remainder of training. Once again, the combination of sparse reward signals and overconfidence in past experience likely caused this premature convergence. Meanwhile, since neural IDAC is minimizing $\rho(\theta)$ instead of $J(\theta)$, it swiftly locates the goal state and finds an optimal policy with average cost $0.1$. This illustrates that, in sparse-reward environments, OIR-based policy gradient methods can lead to improved performance over vanilla techniques.
}

%% file: sections/Conclusion.tex
\section{Conclusion}

%
In this paper we have developed policy gradient methods for a new RL objective, the OIR.
En route, we have elaborated a rich theory underlying these methods, including: a concave programming reformulation of the OIR optimization problem with links to the powerful linear programming theory for MDPs; policy gradient theorems for the OIR setting; and both asymptotic and non-asymptotic convergence theory with global optimality guarantees \newstuff{under appropriate assumptions}.
%
%
\newstuff{We have furthermore presented empirical results that indicate promising performance compared with state-of-the-art methods on sparse-reward problems.}
Interesting directions for future work include extensions to more general classes of ratio optimization problems, development of variants of the IDAC algorithm for continuous spaces using suitable density estimation techniques,
\newstuff{exploration of whether the OIR enables faster-than-linear non-asymptotic rate analyses,}
and thorough empirical evaluation of deep RL variants of IDAC on a range of benchmark problems.
